\documentclass[twoside]{article}
\usepackage[accepted]{aistats2018}
\usepackage{amsmath, amssymb, amsthm}
\usepackage{algorithm}
\usepackage{algorithmic}
\usepackage{graphicx}
\usepackage{tabularx}
\usepackage{tikz}

\newtheorem{Theorem}{Theorem}
\newtheorem{Lemma}[Theorem]{Lemma}
\newtheorem{Proposition}[Theorem]{Proposition}
\newtheorem{Corollary}[Theorem]{Corollary}
\newtheorem{Assumption}{Assumption}

\newtheorem{Definition}{Definition}

\newcommand{\EE}{\mathrm{E}}

\newcommand{\Real}{\mathbb{R}}
\newcommand{\calB}{\mathcal{B}}
\newcommand{\GRE}{\mathrm{GRE}}
\newcommand{\betastar}{\beta^*}
\newcommand{\betahat}{\hat{\beta}}
\newcommand{\betatilde}{\tilde{\beta}}
\newcommand{\Llambda}{L_{\lambda_n}}
\DeclareMathOperator{\supp}{supp}
\DeclareMathOperator{\sgn}{sgn}

\newcommand{\betacheck}{\check{\beta}}
\newcommand{\betabar}{\bar{\beta}}

\usepackage{color}

\begin{document}

%

%

\twocolumn[

\aistatstitle{Independently Interpretable Lasso: A New Regularizer
for Sparse Regression with Uncorrelated Variables}

\aistatsauthor{
	Masaaki Takada${}^{1, 2}$
	\And
	Taiji Suzuki${}^{3, 4, 5}$
	\And
	Hironori Fujisawa${}^{1, 6}$
}
\aistatsaddress{
	\texttt{mtakada@ism.ac.jp}
	\And
	\texttt{taiji@mist.i.u-tokyo.ac.jp}
	\And
	\texttt{fujisawa@ism.ac.jp}
}
\aistatsaddress{
	${}^{1}$School of Advanced Sciences, The Graduate University for Advanced Studies\\
	${}^{2}$Toshiba Corporation\\
	${}^{3}$Graduate School of Information Science and Technology, The University of Tokyo\\
	${}^{4}$PRESTO, Japan Science and Technology Agency, Japan\\
	${}^{5}$Center for Advanced Integrated Intelligence Research, RIKEN, Tokyo, Japan\\
	${}^{6}$The Institute of Statistical Mathematics
}

]

\begin{abstract}
Sparse regularization such as $\ell_1$ regularization is
a quite powerful and widely used strategy for high dimensional learning problems.
The effectiveness of sparse regularization has been supported
practically and theoretically by several studies.
However, one of the biggest issues in sparse regularization is that
its performance is quite sensitive to correlations between features.
Ordinary $\ell_1$ regularization can select variables correlated with each other,
which results in deterioration of not only its generalization error but also interpretability.
In this paper, we propose a new regularization method, ``Independently Interpretable Lasso'' (IILasso).
Our proposed regularizer suppresses selecting correlated variables, and thus each active variable independently affects the objective variable in the model.
Hence, we can interpret regression coefficients intuitively
and also improve the performance by avoiding overfitting.
We analyze theoretical property of IILasso and show that the proposed method is much advantageous for its sign recovery and achieves almost minimax optimal convergence rate.
Synthetic and real data analyses also indicate the effectiveness of IILasso.
\end{abstract}

\section{Introduction}
\label{sec:Introduction}
High dimensional data appears in many fields such as biology, economy and industry.
A common approach for high dimensional regression is
sparse regularization strategy such as
Lasso (Least absolute shrinkage and selection operator) \cite{Tibshirani:1996}.
Since the sparse regularization performs both parameter estimation and feature selection simultaneously,
(i) it offers {\it interpretable} results by identifying informative variables,
and (ii) it can effectively {\it avoid overfitting} by discarding redundant variables.
Because of these properties, sparse regularization has shown huge success in wide range of data analysis in science and engineering.
Moreover, several theoretical studies have been developed
to support the effectiveness of sparse regularization,
and efficient optimization methods also have been proposed
so that sparse learning is easily executed.

However, the performance of sparse regularization is guaranteed
only under ``small correlation'' assumptions, that is,
the features are not much correlated with each other.
Actually, typical theoretical supports are based on a kind of
small correlation assumptions such as {\it restricted eigenvalue condition}
\cite{bickel2009simultaneous,Bhlmann:2011:SHD:2031491}.
In the situation where the features are highly correlated,
the selected variables are likely to be jointly correlated.
As a result, each coefficient cannot be seen as independent variable contribution
so that the model is no longer easy to interpret.
This kind of interpretability is called ``decomposability'' in \cite{Lipton:2016},
which represents the ability whether we can decompose a model into some parts and interpret each component.
In our experience, this kind of interpretability is quite important in many practical modeling for decision-making.
In addition to this, selecting several correlated variables results in worse generalization error
because the redundant representation tends to give overfitting.
Thus, it is favorable 
to construct the model with uncorrelated variables both for
interpretability and generalization ability.


Several methods have been proposed to resolve the problem induced by correlations among variables.
The first line of research is based on a strategy in which correlated variables are either all selected or not selected at all.
Examples of this line are Elastic Net \cite{Zou:2005}, Pairwise Elastic Net \cite{Lobert:2010} and Trace Lasso \cite{Grave:2011}.
These methods select not only important variables but also variables strongly correlated. 
Although these methods  often give better generalization error,
this strategy makes it hard to interpret the model, because many correlated variables are incorporated into the final model.

The second line of research including our proposed method is based on another strategy in which we select uncorrelated variables, and thus obtain decomposability.
Uncorrelated Lasso \cite{Chen:2013} intends to construct a model with uncorrelated variables.
However, it still tends to select ``negatively'' correlated variables
and hence the correlation problem is not resolved.
Exclusive Group Lasso \cite{Kong:2014} is also in this line, but it is necessary to group correlated variables beforehand.
They suggest that we group variables whose correlations are greater than a certain threshold.
However, determination of the threshold is not a trivial problem and
practically it causes unstable results.

In this paper, we propose a new regularization method, named ``Independently Interpretable Lasso'' (IILasso),
which offers efficient variable selection,
does not select negatively correlated variables, and
is free from a specific pre-processing such as grouping.
Our proposed regularization formulation
more aggressively induces the sparsity of the active variables and
reduces the correlations among them.
Hence, we can independently interpret the effects of active variables in the output model, and the generalization performance is also much improved.
To support the usefulness of our proposal,
we give the following contributions:
\begin{itemize}
	\item We show the necessary and sufficient condition for the sign consistency of variables selection.
	As a result, it is shown that our method achieves the sign consistency
	under a milder condition than Lasso.
	\item The convergence rate of the estimation error is analyzed.
	We show that the estimation error achieves the almost minimax optimal rate
	and gives better performance than Lasso in some situations.
	\item We propose a coordinate descent algorithm to find a local optimum of the objective function. This is guaranteed to converge to a stationary point.
	\item It will be shown that
	every local optimal solution achieves the same statistical error rate
	as the global optimal solution and thus is almost minimax optimal
	though the objective function is not necessarily convex
	for the sake of better statistical properties.
\end{itemize}

The rest of this paper is as follows:
In Section 2, we propose a new regularization formulation and introduce its optimization method.
We also state the relationship with existing works.
In Section 3, we show theoretical results on the sign recovery and the convergence rate.
We can see that IILasso is superior to Lasso for correlated design.
In Section 4, both synthetic and real-world data experiments, including 10 microarray datasets, are illustrated.
In Section 5, we summarize the properties of IILasso.

{\bf Notations}.
Let $M \in \mathbb{R}^{n \times p}$.
We use subscripts for the columns of $M$, i.e., $M_j$ denotes the $j$-th column.
Let $v \in \mathbb{R}^{p}$.
${\rm Diag}(v) \in \mathbb{R}^{p \times p}$ is the diagonal matrix whose $j$-th diagonal element is $v_j$.
$|v|$ is the element-wise absolute vector whose $j$-th element is $|v_j|$.
${\rm sgn}(v)$ is the sign vector whose elements are $1 \ {\rm for} \ v_j>0$, $-1 \ {\rm for} \ v_j < 0$, ${\rm and} \ 0 \ {\rm for} \ v_j=0$.
$\supp(v)$ is the support set of $v$, i.e., $\{ j \in \{ 1, \cdots, p \}| v_j \neq 0 \}$.
Let $S$ be a subset of $\{ 1, \cdots, p \}$.
$|S|$ is the number of the elements in $S$.
$S^c$ is the complement subset of $S$, i.e., $S^c = \{ 1, \cdots, p \} \backslash S $.
$v_S$ is the vector $v$ restricted to the index set $S$.
$M_{S_1 S_2}$ is the matrix whose row indexes are restricted to $S_1$ and column indexes are restricted to $S_2$.

\section{Proposed Method}

\subsection{IILasso: A New Regularization Formulation}


\begin{figure*}[t]
	\begin{center}
		\includegraphics[scale=0.48]{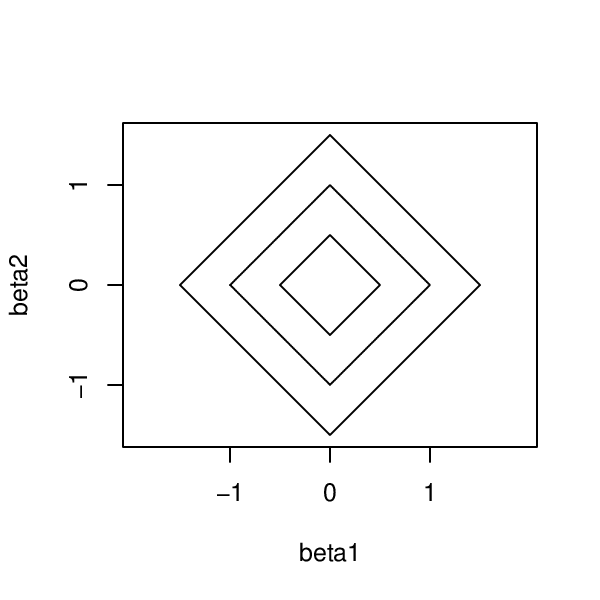}
		\includegraphics[scale=0.48]{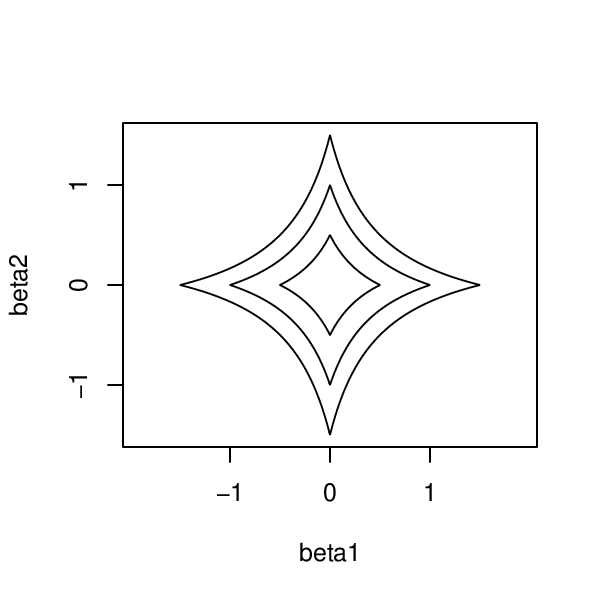}
		\includegraphics[scale=0.48]{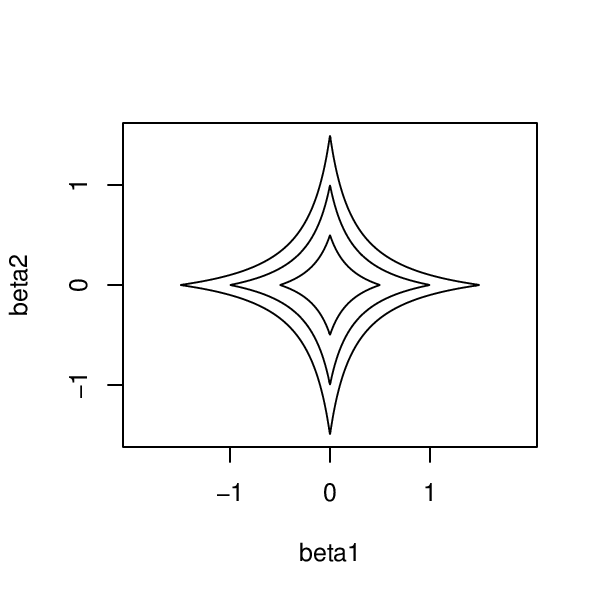}
		\includegraphics[scale=0.48]{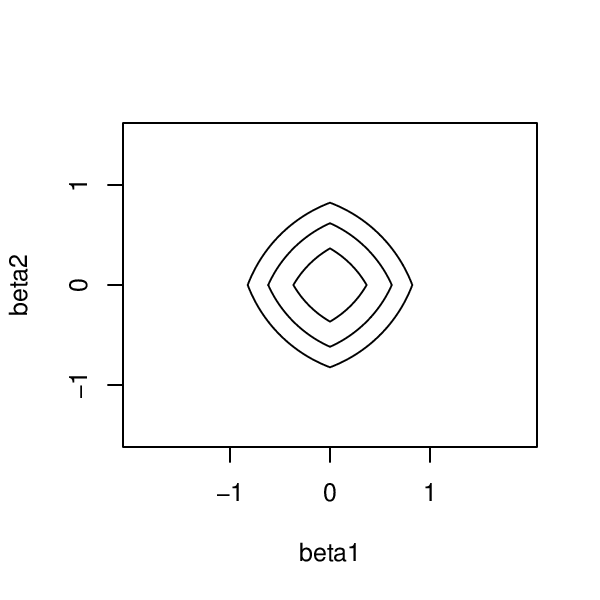}
		\includegraphics[scale=0.48]{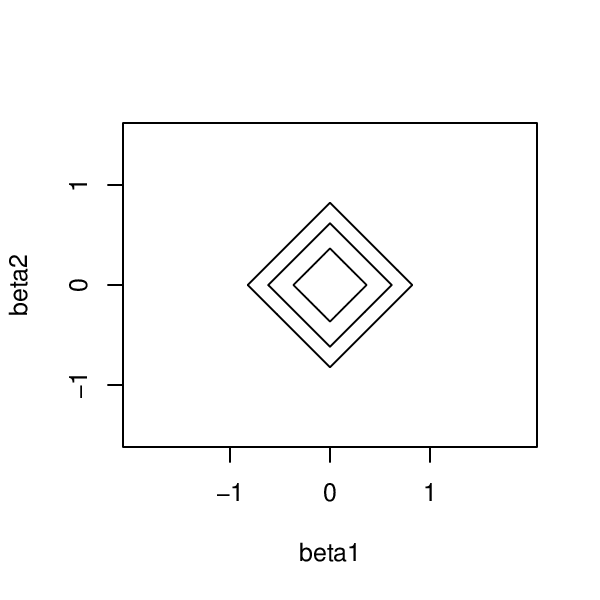}
		\includegraphics[scale=0.48]{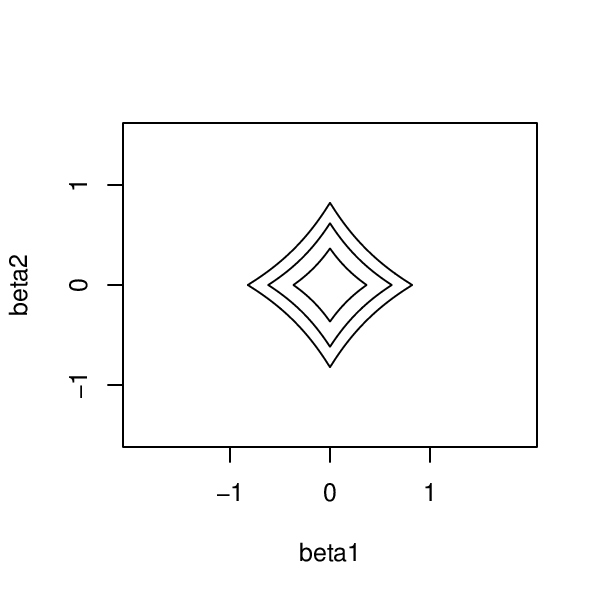}
		\caption{Contours of penalties such that $\|\beta\|_1 + \frac{1}{2} |\beta|^T R |\beta| = 0.5,~ 1,~ 1.5$ with $\beta = [\beta_1, \beta_2]$, $
			R = [0, 0; 0, 0], ~
			[0, 1; 1, 0], ~
			[0, 2; 2, 0]
		$ in the top panel, and $
			R = [1, 0; 0, 1], ~
			[1, 1; 1, 1], ~
			[1, 2; 2, 1]
		$ in the bottom panel from left to right.}
		\label{fig:Contours}
		\end{center}
\end{figure*}

Consider the problem of predicting $y \in \mathbb{R}^n$, given a design matrix $X \in \mathbb{R}^{n\times p}$, assuming a linear model
\begin{align*}
y = X \beta + \epsilon,
\end{align*}
where $\epsilon \in \mathbb{R}^n$ is a noise and $\beta \in \mathbb{R}^p$ is a regression coefficient.
We assume without loss of generality that the features are standardized such that $\Sigma_{i=1}^n X_{ij} = 0, \ \Sigma_{i=1}^n X_{ij}^2 / n = 1$ and $\Sigma_{i=1}^n y_{i} = 0$.
Then, Lasso solves the following optimization problem:
\begin{align*}
\min_\beta \frac{1}{2n}\|y - X \beta\|_2^2 + \lambda ||\beta\|_1,
\end{align*}
where $\lambda > 0$ is a regularization parameter.
Since the last penalty term induces the sparsity of estimated values,
the output model contains few variables and hence is tractable.
This is critically important for interpretability.
However, as we have described in Section \ref{sec:Introduction},
when there are highly correlated variables,
Lasso can select correlated variables especially for a small $\lambda$
resulting in worse interpretability and poor generalization error.

To overcome this problem, we propose a new regularization formulation as follows:
\begin{align}
\min_\beta \frac{1}{2n}\|y - X \beta\|_2^2 + \lambda \left( \|\beta\|_1 + \frac{\alpha}{2} |\beta|^\top R |\beta| \right), \label{eq:objective}
\end{align}
where $\alpha > 0$ is a regularization parameter for the new regularization term, and $R \in \mathbb{R}^{p \times p}$ is a symmetric matrix whose component $R_{jk} \geq 0$ represents the similarity between $X_j$ and $X_k$.
The last term of (\ref{eq:objective}) is also written as $\frac{\lambda \alpha}{2} \sum_{j=1}^p \sum_{k=1}^p R_{jk}|\beta_j||\beta_k|$.
We define $R_{jk}$ for $j \neq k$ as a monotonically increasing function of the absolute correlation $r_{jk}=\frac{1}{n}|X_j^\top X_k|$,
so that correlated variables are hard to be selected simultaneously.
In particular, when $X_j$ and $X_k$ are strongly correlated, the squared error does not change under the condition that $\beta_j + \beta_k$ is constant, but the penalty $R_{jk} |\beta_j| |\beta_k|$ strongly induces either $\beta_j=0$ or $\beta_k=0$.
On the contrary, if $X_j$ and $X_k$ are uncorrelated, i.e., $R_{jk}$ is small, then the penalty of selecting both $\beta_j$ and $\beta_k$ is negligible.

We can see the exclusive effect of our regularization term by the constraint regions corresponding to the penalties.
Figure \ref{fig:Contours} illustrates the constraint regions of $\|\beta\|_1 + \frac{1}{2} |\beta|^T R |\beta|$ for the case $p=2$.
As diagonal elements of $R$ increases (from the top to the bottom panel), the contours become smooth at the axis of coordinates.
Because of this, the minimizer tends to select both variables if two variables are strongly correlated.
This is the grouping effect of Elastic Net as we describe later.
On the other hand, as off-diagonal elements of $R$ increases (from the left to the right panel), the contours become pointed at the axis of coordinates and the minimizer tends to be sparser.
This is the exclusive effect for correlated variables.
Although the shape of contours resemble that of $\ell_q \ (0<q<1)$ penalty, we would emphasize that we use correlation information $R_{jk}$ and multiplication term $|\beta_j| |\beta_k|$.
Since our regularization term is adaptive for correlations, our penalty achieves both sparse and stable solutions.

Some definition variations of the similarity matrix $R$ can be considered.
One of the natural choices is $R_{jk} = r_{jk}^2$.
$R$ is positive semidefinite in this case because the Hadamard product of positive semidefinite matrices is also positive semidefinite.
Hence, the problem (\ref{eq:objective}) turns to be convex and easy to solve the global optimum.
However, it may not reduce correlations enough.
Another choice is $R_{jk} = |r_{jk}|$, which reduces correlations more strongly.
Another effective choice is $R_{jk} = |r_{jk}| / (1 - |r_{jk}|)$ for $j \neq k$ and $R_{jk} = 0$ for $j=k$.
In this case, if a correlation between two certain variables becomes higher, i.e., $r_{jk} \rightarrow 1$, then the penalty term diverges infinitely and IILasso cannot simultaneously select both of them.
We use the last one in our numerical experiments, because it is favorable from theoretical studies, as described in Section~\ref{sec:theory}.

Now, let us emphasize the advantage of uncorrelated models for interpretability
with a simple example (Figure \ref{fig:example}).
Suppose $X=[X_1,X_2,X_3] \in \mathbb{R}^{n \times 3}$ is standardized,
$X_1$ and $X_2$ are orthogonal, and $X_3=(X_1+X_2)/\sqrt{2}$.
Consider two models: (A) $y=2X_1+X_2$ and (B) $y=X_1+\sqrt{2}X_3$.
Both models output the same prediction.
Which model do you think is more interpretable?
We believe that the model (A) is more interpretable than (B),
because active variables are uncorrelated so that we can see each coefficient as independent variable contribution.
Although Lasso selects (B) because $\ell_1$ norm of its coefficients is small,
IILasso for $\alpha$ large enough selects (A) because our reguralization term excludes correlations.

\begin{figure}[t]
	\begin{center}
		\includegraphics[scale=0.55]{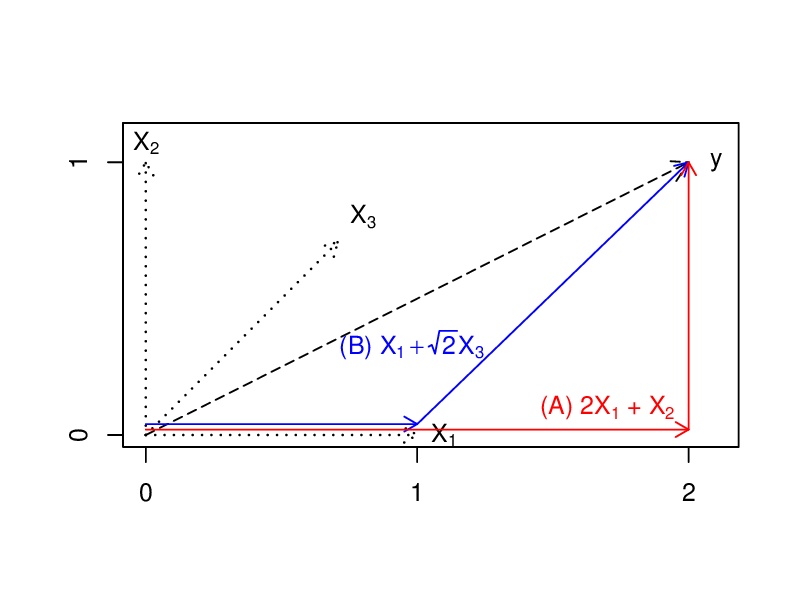}
		\caption{Uncorrelated model vs correlated model}
		\label{fig:example}
		\end{center}
\end{figure}

\subsection{Optimization}
To solve (\ref{eq:objective}), we introduce Coordinate Descent Algorithm (CDA),
which was originally proposed for Lasso ($\alpha = 0$ for IILasso) \cite{Friedman:2007, Friedman:2010}.
CDA is a simple and efficient algorithm, particularly in high dimensions.
CDA basically follows simply: For each $j \in \{1, \cdots, p\}$, we optimize the objective function with respect to $\beta_j$ with the remaining elements of $\beta$ fixed at their most recently updated values.

CDA is applicable even when the quadratic penalty is included.
Let $L(\beta)$ denote the objective function in (\ref{eq:objective}).
To derive the update equation, when $\beta_j \neq 0$, differentiating $L(\beta)$ with respect to $\beta_j$ yields
\begin{align*}
\partial_{\beta_j} L(\beta) =&
- \frac{1}{n} X_j^\top \left( y - X_{-j} \beta_{-j} \right)
+ \left(1 + \lambda \alpha R_{jj}\right) \beta_j \\
&+ \lambda \left( 1 + \alpha R_{j, -j} |\beta_{-j}| \right) {\rm sgn}(\beta_j),
\end{align*}
where $\beta_{-j}$ denotes $\beta$ without the $j$-th component, $X_{-j}$ denotes $X$ without $j$-th column and $R_{j, -j}$ denotes the $j$-th row vector without $j$-th column of $R$.
Solving $\partial_{\beta_j} L(\beta) = 0$, we obtain the update rule as
\begin{align}
\beta_j \leftarrow \frac{1}{1+\lambda \alpha R_{jj}}
S\biggl( & \frac{1}{n} X_j^\top \left( y - X_{-j} \beta_{-j} \right), \notag \\
& \lambda \left( 1 + \alpha R_{j, -j} |\beta_{-j}| \right) \biggr), \label{eq:cdaupdate}
\end{align}
where $S(z, \gamma)$ is a soft thresholding function
\begin{align*}
S(z, \gamma) :&= {\rm sgn}(z)(|z|-\gamma)_+ \\
&= \begin{cases}
z - \gamma & {\rm if} \ z>0 \ {\rm and} \ \gamma < |z|,\\
z + \gamma & {\rm if} \ z<0 \ {\rm and} \ \gamma < |z|,\\
0 & {\rm if} \ |z| \leq \gamma.
\end{cases}
\end{align*}

The whole algorithm for solving IILasso is described in Algorithm \ref{alg:iilasso}.
We search several $\lambda$ from $\lambda_{\max}$ to $\lambda_{\min}$.
$\beta$ is initialized at each $\lambda$ in some way such as a) zeros for all elements, b) the solution of previous $\lambda$, or c) the solution of ordinary Lasso.

\begin{algorithm}[t]
	\caption{CDA for IILasso}
	\label{alg:iilasso}
	\begin{algorithmic}
		\FOR{$\lambda = \lambda_{\max}, \cdots, \lambda_{\min}$}
		\STATE initialize $\beta$
		\WHILE{until convergence}
		\FOR{$j = 1, \cdots, p$}
		\STATE $\beta_j \leftarrow \frac{1}{1+\lambda \alpha R_{jj}}
		S\bigl(\frac{1}{n} X_j^\top \left( y - X_{-j} \beta_{-j} \right), $
		\\
		$~~~~~~~~~~~~~~~~~~~~~~ \lambda \left( 1 + \alpha R_{j, -j} |\beta_{-j}| \right) \bigr)$
		\ENDFOR
		\ENDWHILE
		\ENDFOR
	\end{algorithmic}
\end{algorithm}

In Algorithm \ref{alg:iilasso}, we can see that the objective function monotonically decreases at each update and the estimate converges a stationary point.

\begin{Proposition}
	Let $\{\beta^t\}_{t=0,1,\cdots}$ be a sequence of $\beta$ in Algorithm \ref{alg:iilasso}.
	Then, every cluster point of $\{ \beta^t \}_{t \equiv (p-1) {\rm mod} p}$ is a stationary point.
\end{Proposition}
\begin{proof}
The proof is based on Theorem 4.1 in \cite{Tseng:2001}.
First, we can see that the level set $\{ \beta | L(\beta) \leq L(\beta^0) \}$ is compact and $L(\beta)$ is continuous.
Moreover, $L(\beta)$ has a unique minimum with (\ref{eq:cdaupdate}) in terms of $\beta_j$.
Therefore, every cluster point of $\{ \beta^t \}_{t \equiv (p-1) {\rm mod} p}$ is a coordinatewise minimum point.
In addition, since $L(\beta)$ can be seen as a locally quadratic function in any directions,
$L(\beta)$ is {\it regular} at the cluster point.
These and Theorem 4.1 (c) in \cite{Tseng:2001} concludes the assertion.
\end{proof}



\subsection{Related Work}
There are some existing works that take correlations among variables into account.
We can divide them into mainly two directions: 1) grouping selection and 2) exclusive selection.
The former groups correlated variables and selects correlated variables together.
The latter excludes correlated variables and selects uncorrelated variables.

The representative method of grouping selection is Elastic Net \cite{Zou:2005}.
The objective function is constructed by squared $\ell_2$ penalty in addition to $\ell_1$ penalty:
\begin{align*}
\min_\beta ||y - X \beta\|_2^2 + \lambda_1 \|\beta\|_1 + \lambda_2 \|\beta\|_2^2.
\end{align*}
Due to $\ell_2$ penalty, if the variables $X_j$ and $X_k$ are strongly correlated tending to be 1, then estimated values of $\beta_j$ and $\beta_k$ get closer.
If $X_j$ and $X_k$ are equal, especially, then $\beta_j$ and $\beta_k$ must be equal.
This behavior is called grouping effect.
Pairwise Elastic Net \cite{Lobert:2010} and Trace Lasso \cite{Grave:2011} are the same direction of research (and improve the prediction accuracy).
These methods tend to include many correlated variables and each coefficient no longer indicates independent variable contribution.
As a result, it is hard to interpret which variables are truly active and how variables affect the objective variable.

Another direction of research is exclusive selection.
Uncorrelated Lasso (ULasso) \cite{Chen:2013} aims to reduce correlations among active variables.
It optimizes the following objective function:
\begin{align}
\min_\beta \|y - X \beta\|_2^2 + \lambda_1 \|\beta\|_1 + \lambda_2 \beta^\top R \beta, \label{eq:uncorrelated}
\end{align}
where $R \in \mathbb{R}^{p\times p}$ with each element $R_{jk} = (\frac{1}{n}X_j^\top X_k)^2$.
ULasso quite resembles our formulation, but there exists a critical difference that they use $\beta$ instead of $|\beta|$ in the objective function (\ref{eq:uncorrelated}).
We found that ULasso does not necessarily select uncorrelated variables.
For example, consider the case $X=[X_1, X_2]$.
The last term of (\ref{eq:uncorrelated}) is $\lambda_2 (\beta_1^2 + \beta_2^2 + 2R_{12} \beta_1\beta_2)$.
If $R_{12} \neq 0$, then the term $R_{12}\beta_1\beta_2$ encourages $|\beta_1\beta_2|$ larger with $\beta_1\beta_2<0$.
This implies that ULasso tends to select correlated variables and set coefficients to the opposite sign.
In particular, $X_1$ and $X_2$ are strongly correlated, then it reduces $\lambda_2(\beta_1+\beta_2)^2$, which induces $\beta_1 = - \beta_2$.
It is not a major problem when $X_1$ and $X_2$ are positively correlated, but is a significant problem when $X_1$ and $X_2$ are negatively correlated.
This problem is overcome in our method, as described in Section 2.1.
Therefore, the difference between their ULasso and our IILasso is essential and crucial.

Excusive Group Lasso (EGLasso) \cite{Kong:2014} is also the same direction of exclusive selection.
It optimizes the following objective function:
\begin{align}
\min_\beta \|y - X \beta\|_2^2  + \lambda_1 \|\beta\|_1 + \lambda_2 \sum_{k=1}^K \|\beta^{(k)}\|_1^2, \label{eq:eglasso}
\end{align}
where $\beta^{(k)}$ consists of the variables of $\beta$ within a group of predictors $g_k \subset \{1, \cdots, p\}$ and $K$ is the number of groups.
The last term is $\ell$1/$\ell$2 penalty, which acts on exclusive feature selection.
For example, when $p=3, g_1=\{1,2\}$ and $g_2=\{3\}$, then the last term becomes $\lambda_2 ((|\beta_1|+|\beta_2|)^2 + |\beta_3|^2)$.
This enforces sparsity over each intra-group.
They suggest that we put highly correlated variables into the same group in order to select uncorrelated variables.
They use $|r_{ij}|>\theta$ with $\theta \approx 0.90$ as a threshold.
EGLasso can be seen as a special case of IILasso.
Let $R$ be a group indicator matrix such as $R_{jk}=1$ if $X_j$ and $X_k$ belong to the same group and $R_{jk}=0$ otherwise.
Then IILasso is reduced to EGLasso.
For the above example, if we define similarity matrix $R=[1, 1, 0; 1, 1, 0; 0, 0, 1]$,
then the last term of IILasso objective function (\ref{eq:objective}) becomes $\lambda \left( \beta_1^2 + 2|\beta_1||\beta_2| + \beta_2^2 + \beta_3^2 \right)$, which is the same as the last term of (\ref{eq:eglasso}).
As we see, EGLasso needs to determine the threshold $\theta$ and group variables beforehand.
This can cause severely unstable estimation.

\section{Theoretical Results}
\label{sec:theory}
In this section, we show theoretical properties of IILasso.
We first show the sign recovery condition of IILasso.
Then, we derive the convergence rate of IILasso.
These results are significantly important for interpretability and generalization ability.
In addition, we show the property of local minimum, which implies that every local optimal solution achieves the same statistical error rate as the global optimal solution.
In this section, let $\betastar$ denote the true parameter, $S$ denote the true active sets, i.e., $S = \{ j \in \{ 1, \cdots, p \} | \beta_j^* \neq 0 \}$, and $s = |S|$.

\subsection{Sign Recovery}

We give the necessary and sufficient condition of sign recovery.
\begin{Theorem} \label{theorem:sign}
	Define
	\begin{align*}
	U :=& \frac{1}{n} X_S^\top X_S + \lambda \alpha {\rm Diag}({\rm sgn}(\beta_S^*)) R_{SS} {\rm Diag}({\rm sgn}(\beta_S^*)), \\
	V :=& \lambda {\rm sgn}(\beta_S^*) + \lambda \alpha{\rm Diag}({\rm sgn}(\beta_S^*)) R_{SS} {\rm Diag}({\rm sgn}(\beta_S^*)) \beta_S^*\\
	&- \frac{1}{n} X_S^\top \epsilon.
	\end{align*}
	Assume $U$ is invertible.
	Then, there exists a critical point $\hat{\beta}$ of (\ref{eq:objective}) with correct sign recovery ${\rm sgn}(\hat{\beta}) = {\rm sgn}(\beta^*)$ if and only if the following two conditions hold:
	\begin{align}
	&{\rm sgn} \left( \beta_S^* - U^{-1} V \right) = {\rm sgn}(\beta_S^*),\label{eq:kkts} \\
	&\left| \frac{1}{n}X_{S^c}^\top X_S U^{-1} V
	+ \frac{1}{n} X_{S^c}^\top \epsilon \right| \notag \\
	&\leq
	\lambda \left(1 + \alpha R_{S^cS} \left| \beta_S^* - U^{-1} V \right| \right), \label{eq:kktsc}
	\end{align}
	where both of these vector inequalities are taken elementwise.
\end{Theorem}

The proof is given in the supplementary material.
The sign recovery condition is derived from the standard conditions for optimality.
We note that $\alpha = 0$ reduces the condition into the ordinary Lasso condition in \cite{Wainwright:2009}.
The invertible assumption of $U$ is not restrictive because it is true for almost all $\lambda$ if $X_S^T X_S$ is invertible, which is the same assumption as standard analysis of Lasso.

The condition of IILasso is milder than that of Lasso when $R_{SS}$ is small enough, since (\ref{eq:kkts}) is the same as Lasso and (\ref{eq:kktsc}) is easier to be satisfied unless $\alpha R_{S^cS} = 0$.
This implies that IILasso is more favorable than Lasso from the viewpoint of sign recovery. 
In addition, we can see that a large value of $R_{S^cS}$ is favorable for sign recovery.
That is, IILasso tends to succeed the sign recovery if the true active variables have small correlations and the inactive variables are strongly correlated with the true active variables.

Theorem \ref{theorem:sign} is {\it not} the condition of a global optimal solution.
However, if global optimal solutions are finite, they must be a critical point.
Hence, there exists a global optimal solution with correct sign recovery only if (\ref{eq:kkts}) and (\ref{eq:kktsc}) hold.




\subsection{Convergence Rate}

Here we give the convergence rate of the estimation error of our method.
Before we give the statement, the assumption and definition are prepared.
\begin{Assumption}
	\label{ass:Noise}
	\rm
	$(\epsilon_i)_{i=1}^n$ is an i.i.d. sub-Gaussian sequence: $\EE[e^{t\epsilon_i}]\leq e^{\frac{\sigma^2 t^2}{2}}~(\forall t \in \Real)$ for $\sigma > 0$.
\end{Assumption}


\begin{Definition}[Generalized Restricted Eigenvalue Condition ($\GRE(S,C,C')$)]
\label{eq:AssGRE}
Let a set of vectors $\calB(S, C, C')$ for $C>0, C'>0$ be
\begin{align*}
	\calB(S, C, C') := \{ & \beta \in \Real^p \mid
	 \| \beta_{S^c} \|_1
	 + \frac{C'\alpha}{2} |\beta_{S^c}|^\top R_{S^cS^c}|\beta_{S^c}|\\
	& + C' \alpha |\beta_{S^c}|^\top R_{S^cS} |\beta_S + \betastar_S|
	\leq  C \| \beta_S \|_1
	\}.
\end{align*}
Then, we assume $\phi_{\GRE} >0$ where
$$
\phi_{\GRE} = \phi_{\GRE}(S, C, C') := \inf_{ v \in \calB(S, C, C') } \frac{v^\top \frac{1}{n} X^\top X v}{\|v\|_2^2}.
$$
\end{Definition}
Definition \ref{eq:AssGRE} is a generalized notion of the {\it restricted eigenvalue condition} \cite{bickel2009simultaneous,Bhlmann:2011:SHD:2031491}
tailored for our regularization.
One can see that, if $\alpha = 0$ or $C' = 0$, then $\phi_{\GRE}$ is reduced to the ordinary restricted eigenvalue \cite{bickel2009simultaneous,raskutti2010restricted}
for the analysis of Lasso.
Since there are additional terms related to $|\beta|^\top R |\beta|$, the set $\calB(S, C, C')$ is smaller than that for the ordinary restricted eigenvalue if the same $C$ is used.
In particular, the term $|\beta_{S^c}|^\top R_{S^cS} |\beta_S + \betastar_S|$ on the left-hand side strongly restricts the amplitude of coefficients for unimportant variables (especially the variables with large $R_{S^cS}$), and thus promote independence among the selected variables.

GRE condition holds in quite general class of Gaussian design, because RE$(S, C)$ condition is satisfied in general class of Gaussian design \cite{Raskutti:2010}, and Assumption GRE$(S, C, C')$ is milder than Assumption RE$(S, C)$.
Therefore, GRE condition is not so restrictive.


Here, for any $\delta > 0$, we define $\gamma_n$ as
$$
\gamma_n  =\gamma_n(\delta) := \sigma  \sqrt{\frac{2\log(2p/\delta)}{n}}.
$$
Then, we obtain the convergence rate of IILasso as follows.
\begin{Theorem}
	\label{th:OracleIneaulityBetahat}
	Suppose that Assumption \ref{ass:Noise} is satisfied.
	Suppose $R$ satisfy for all $j, k \in S$,
	$$
	0 \leq R_{ jk} \leq D,
	$$
	for some positive constant $D$,
	and the estimator $\betahat$ of \eqref{eq:objective} is approximately minimizing the objective function so that
	\begin{align}
	\Llambda(\betahat) \leq \Llambda(\betastar).
	\label{eq:betahatcondition}
	\end{align}
	Fix any $0 < \delta < 1$, let the regularization parameters satisfy
	$$
	3\gamma_n \leq \lambda_n \ {\rm and} \ \alpha \leq \frac{1}{4 D \| \betastar_S \|_1}.
	$$
	Suppose that Assumption GRE$(S, 3, \frac{3}{2})$ (Definition \ref{eq:AssGRE}) is satisfied.
	Then, it holds that
		\begin{align*}
		\|\betahat - \betastar\|_2^2
		\leq  \frac{16 s \lambda_n^2}{\phi_{\GRE}^2},
	\end{align*}
	with probability $1-\delta$.
\end{Theorem}

The proof is given in the supplementary material.
The obtained convergence rate is roughly evaluated as
$$
\textstyle \|\betahat - \betastar \|_2^2  = O_p\left(\frac{s \log(p)}{n}\right),
$$
which is almost the minimax optimal rate \cite{raskutti2011minimax}.

As is obvious from the proof of Theorem \ref{th:OracleIneaulityBetahat}, we also have a little bit stricter bound
$$
\|\betahat - \betastar\|_2^2
\leq  \frac{\left( \frac{8}{3} + 5 \alpha D \| \betastar_S \|_1 + \frac{3}{4} (\alpha D \| \betastar_S\|_1)^2 \right)^2 s \lambda_n^2}{\phi_{\GRE}^2},
$$
under Assumption GRE$(S, C, \frac{3}{2})$ (Definition \ref{eq:AssGRE}) where $C=2+\frac{15}{4}\alpha D\|\betastar_S\|_1+\frac{9}{16}(\alpha D\|\betastar_S\|_1)^2$,
with high probability.
This proof is also given in the supplementary material.
We can easily see that, when $\alpha = 0$, then the convergence rate analysis is reduced to the standard one for ordinary Lasso
\cite{bickel2009simultaneous,Bhlmann:2011:SHD:2031491}.
Under well ``interpretable'' cases where the true non-zero components $S$ are independent, i.e., $R_{SS}=O$,
the error bounds for Lasso and IILasso are the same except for the term $\phi_{\rm GRE}^2$.
Since $R_{S^cS}$ and $R_{S^cS^c}$ shrink the set of vectors $\cal{B}(S, C, C')$ in Definition \ref{eq:AssGRE}, $\phi_{\rm GRE}$ of IILasso is larger than that of Lasso.
Therefore, our method achieves a better error bound than the ordinary $\ell_1$ regularization in this setting.
In addition, our method has more advantageous when the variables are correlated between informative and non-informative variables.


\subsection{Local Optimality of estimator}
\newcommand{\Sc}{{V}}
\newcommand{\Ddash}{{\bar{D}}}

The objective function of our method is {\it not} necessarily convex in exchange for better statistical properties as observed above.
Our next theoretical interest is about the global optimality of our optimization algorithm (Algorithm \ref{alg:iilasso}).
Since our optimization method is a greedy one, there is no confirmation that
it achieves the global optimum.
However, as we see in this section, the local solution achieves almost the same estimation error as the global optimum satisfying \eqref{eq:betahatcondition}.
For theoretical simplicity, we assume the following a bit stronger condition.
\begin{Assumption}
	\label{ass:localStrongConvRbound}
	There exists $\phi>0$ and $q_n \leq p$ such that, for all $\Sc \subset \{1,\dots,p\}$ satisfying $|\Sc| \leq q_n$ and $\Sc \cap S = \emptyset$, it holds that
	$$
	\frac{1}{n} X_{S \cup \Sc}^\top X_{S \cup \Sc} \succ \phi \mathrm{I}.
	$$
	Moreover, there exists $\Ddash$ such that the maximum absolute value of the eigenvalue of $R$ is bounded as 
	$$
	\sup_{u \in \Real^{S \cup \Sc}} u^\top ( R_{S \cup \Sc, S\cup \Sc}) u \leq \Ddash \|u\|_2^2.
	$$
\end{Assumption}

\begin{Theorem}
	\label{th:LocalOptimality}
	Suppose Assumptions \ref{ass:Noise} and \ref{ass:localStrongConvRbound} are satisfied.
	Assume that $\betatilde$ is a local optimal solution of \eqref{eq:objective} satisfying
	$|\supp(\betatilde)| \leq |S| + q_n$.
	Let the regularization parameters satisfy
	$$
	\gamma_n < \lambda_n \ {\rm and} \ \alpha < \min \left\{ \frac{\sqrt{s}}{2\Ddash\| \betastar \|_2}, \frac{\phi}{2\Ddash\lambda_n}\right\},
	$$
	for any $\delta > 0$.
	Then, $\betatilde$ should satisfy
	\begin{align*}
		\|\betatilde - \betastar\|_2^2 \leq \frac{ 25 s \lambda_n^2}{\phi^2},
	\end{align*}
	with probability $1-\delta$.
\end{Theorem}

The proof is given in the supplementary material.
Theorem \ref{th:LocalOptimality} indicates that every local optimum $\betatilde$ achieves the same convergence rate with the ideal optimal solution $\betahat$.
In other words, there is {\it no local optimal solution} with sparsity level $|S| + q_n$ far from the true vector $\betastar$.
In the theorem, we assumed the sparsity of  the local optimal solution $\betatilde$.
Such a sparse solution can be easily obtained by running CDA multiple times from different initial solutions.

\section{Numerical Experiments}

\subsection{Synthetic Data}

We consider the case in which the true active variables are uncorrelated and many inactive variables are strongly correlated with the active variable.
If all of active and inactive variables are uncorrelated, it is easy to estimate which is active or inactive.
On the other hand, if the inactive variables are strongly correlated with the active variables, it is hard to distinguish which one is active.
We simulate such a situation and validate the effectiveness of IILasso.

First, we generated a design matrix $X \in \mathbb{R}^{n \times p}$
from the Gaussian distribution of $\mathcal{N}(0, \Sigma)$
where $\Sigma = {\rm Diag}(\Sigma^{(1)}, \cdots, \Sigma^{(b)})$ was a block diagonal matrix
whose element $\Sigma^{(l)} \in \mathbb{R}^{q \times q}$ was  $\Sigma^{(l)}_{jk} = 0.95$ for $j \neq k$ and $\Sigma^{(l)}_{jk} = 1$ for $j = k$.
We set $n = 50$, $p = 100$, $b = 10$ and $q = 10$.
Thus, there were 10 groups containing 10 strongly correlated variables.
Next, we generated an objective variable $y$ by the true active variables $X_1, X_{11}, X_{21}, \cdots, X_{91}$,
such that $y = 10 X_1 - 9 X_{11} + 8 X_{21} - 7X_{31} + \cdots + 2X_{81} - X_{91} + \epsilon$,
with a standard Gaussian noise $\epsilon$.
Each group included one active variable.
We generated three datasets for training, validation and test as above.

Then, we compared the performance of Lasso, SCAD \cite{Zou:2001}, MCP \cite{Zhang:2010}, EGLasso and IILasso.
Evaluation criteria are prediction error (mean squared error),
estimation error ($\ell_2$ norm between the true and estimated coefficients)
and model size (the number of non-zero coefficients).
SCAD and MCP are representative methods of folded concave penalty, so their objective functions are non-convex, which are the same as our method.
They have a tuning parameter $\gamma$; we set $\gamma = 2.5, 3.7, 10, 20, 100, 1000$ for SCAD and $\gamma = 1.5, 3, 10, 20, 100, 1000$ for MCP.
EGLasso has a parameter $\lambda_2$; we set $\lambda_2/\lambda_1 = 0.01, 0.1, 1, 10, 100, 1000$.
For EGLasso, we used the true group information beforehand.
We used {\tt R} packages {\tt ncvreg} \cite{Breheny:2011} for Lasso, SCAD and MCP, and {\tt sparsenet} \cite{Mazumder:2011} for MCP.
One can solve MCP using either {\tt ncvreg} or {\tt sparsenet}, but they differ in their optimization algorithms and ways of initialization.
For IILasso, we defined $R_{jk} = |r_{jk}| / (1 - |r_{jk}|)$ for $j \neq k$ and $R_{jk} = 0$ for $j=k$.
Hence, $R_{SS}$ takes small values if active variables are independent,
and $R_{SS^c}$ and $R_{S^cS^c}$ take large values if inactive variables are strongly correlated with other variables,
which is favorable from the theoretical results obtained in Section~\ref{sec:theory}.
We set $\alpha = 0.01, 0.1, 1, 10, 100, 1000$.
We tuned the above parameters using validation data and calculated errors using test data.
We iterated this procedure 500 times and evaluated the averages and standard errors.

Table \ref{tbl:syntheticdata} shows the performances with their standard error in parentheses.
IILasso achieved the best prediction and estimation among all of them.
This was because our penalty term excluded the correlations and avoided overfitting.
Moreover, the model size of IILasso was much less than those of Lasso and EGLasso, and comparable to MCP.
As a whole, IILasso could distinguish the true active variables.

\begin{table}[t]
	\begin{center}
		\normalsize
		\caption{Results of synthetic data}
		\label{tbl:syntheticdata}
		\small
	\scalebox{0.9}[0.9]{
	\begin{tabular}{|l|r|r|r|} \hline
		& \multicolumn{1}{c|}{prediction} & \multicolumn{1}{c|}{estimation} & \multicolumn{1}{c|}{model} \\
		& \multicolumn{1}{c|}{error} & \multicolumn{1}{c|}{error} & \multicolumn{1}{c|}{size} \\ \hline
		Lasso({\tt ncvreg}) & 2.67(0.05) & 4.44(0.06) & 34.1(0.46) \\ \hline
		SCAD({\tt ncvreg}) & 1.52(0.02) & 1.79(0.04) & 14.6(0.23) \\ \hline
		MCP({\tt ncvreg}) & 1.53(0.02) & 1.79(0.04) & 14.6(0.24) \\ \hline
		MCP({\tt sparsenet}) & 2.41(0.11) & 3.15(0.13) & {\bf 13.4(0.28)} \\ \hline
		EGLasso & 2.60(0.04) & 4.36(0.05) & 33.3(0.32) \\ \hline
		{\bf IILasso~(ours)} & {\bf 1.45(0.02)} & {\bf 1.40(0.04)} & {\bf 13.5(0.23)} \\ \hline
	\end{tabular}
	}
	\end{center}
\end{table}


\subsection{Real Data: Gene Expression Data}


We applied our method to various gene expression data to validate its effectiveness for real applications.
We used the following 10 datasets:
`alon'~\cite{Alon:1999} (colon cancer),
`chiaretti'~\cite{Chiaretti:2004} (leukemia),
`gordon'~\cite{Gordon:2002} (lung cancer),
`gravier'~\cite{Gravier:2010} (breast cancer),
`pomeroy'~\cite{Pomeroy:2002} (central nervous system disorders),
`shipp'~\cite{Shipp:2002} (lymphoma),
`singh'~\cite{Singh:2002} (prostate cancer),
`subramanian'~\cite{Subramanian:2005} (miscellaneous),
`tian'~\cite{Tian:2003} (myeloma),
`west'~\cite{West:2001} (breast cancer).
All of these data are provided by {\tt R} package {\tt datamicroarray}.
The abstract of these datasets is described in Table \ref{tbl:realdata}.
All datasets are small-sample high-dimensional 
DNA microarray data.

Since the objective variable is binary, logistic regression was applied.
Logistic regression of Lasso, SCAD and MCP is supported by {\tt ncvreg} ({\tt sparsenet} does not support logistic regression).
EGLasso and IILasso can also be formulated as logistic regression.
For details, see the supplementary material.
We used the same settings on regularization parameters as described in Section 4.1.
We evaluated the log-likelihood, misclassification error and model size using ten-fold cross validation.

The results are given in Figure \ref{fig:realdata}.
From the viewpoint of log-likelihood, IILasso won in 5 out of 10 cases.
MCP showed similar performance to IILasso, but only won 2 cases.
Lasso, EGLasso and SCAD showed similar performance, but fell behind IILasso and MCP.
We can see the similar tendency of misclassification errors.
IILasso won in 9 out of 10 cases including 5 ties.
IILasso lost only 1 case, in which MCP won.
In addition, the model size of IILasso was always smaller than Lasso and the smallest in 6 out of 10 cases including 1 tie.
In particular, IILasso showed a much smaller misclassification error with a much smaller model size for the dataset `subramanian', and smaller model sizes for the datsets `gordon', `gravier', `pomeroy' and `west' among the comparable methods with almost the same misclassification errors.
As a whole, IILasso could construct accurate models with a small number of variables.

\begin{table}[t]
	\begin{center}
		\normalsize
		\caption{Abstract of microarray data}
		\label{tbl:realdata}
		\small
		\begin{tabular}{|l|r|r|}\hline
			data & \# samples & \# dimensions \\ \hline
			alon & 62 & 2000 \\ \hline
			chiaretti & 111 & 12625 \\ \hline
			gordon & 181 & 12533 \\ \hline
			gravier & 168 & 2905 \\ \hline
			pomeroy & 60 & 7128 \\ \hline
			shipp & 58 & 6817 \\ \hline
			singh & 102 & 12600 \\ \hline
			subramanian & 50 & 10100 \\ \hline
			tian & 173 & 12625 \\ \hline
			west & 49 & 7129 \\ \hline
		\end{tabular}
	\end{center}
\end{table}

\begin{figure*}[t]
	\begin{tabular}{ccc}
		\begin{minipage}[t]{0.3\linewidth}
			\includegraphics[keepaspectratio,scale=0.35]{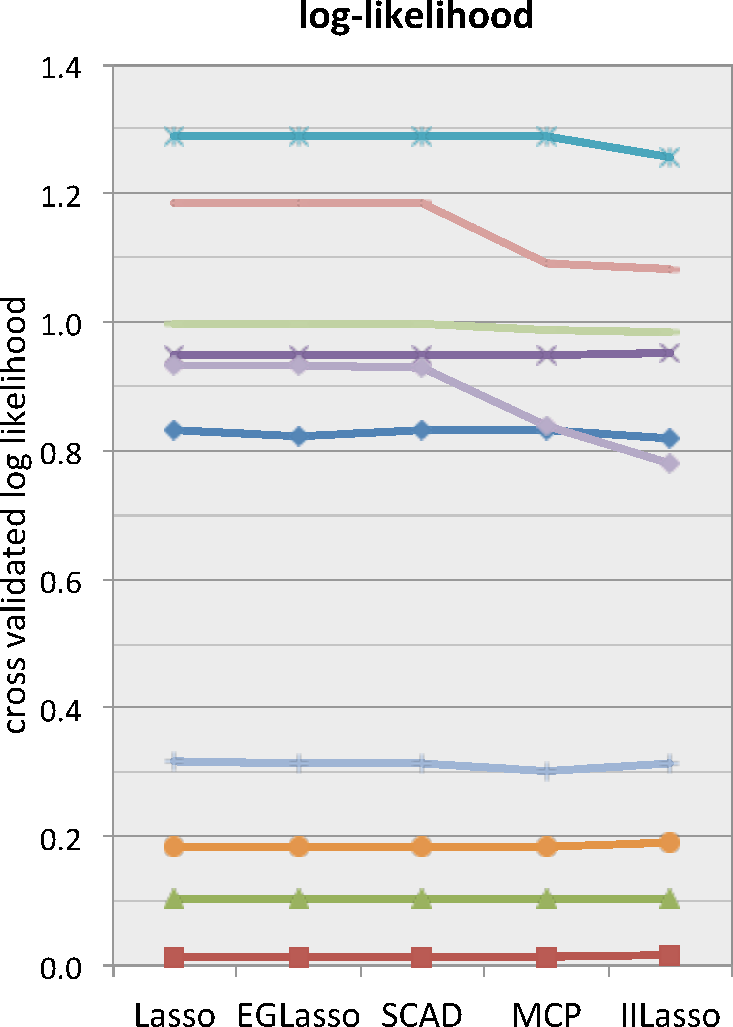}
		\end{minipage}
		\begin{minipage}[t]{0.3\linewidth}
			\includegraphics[keepaspectratio,scale=0.35]{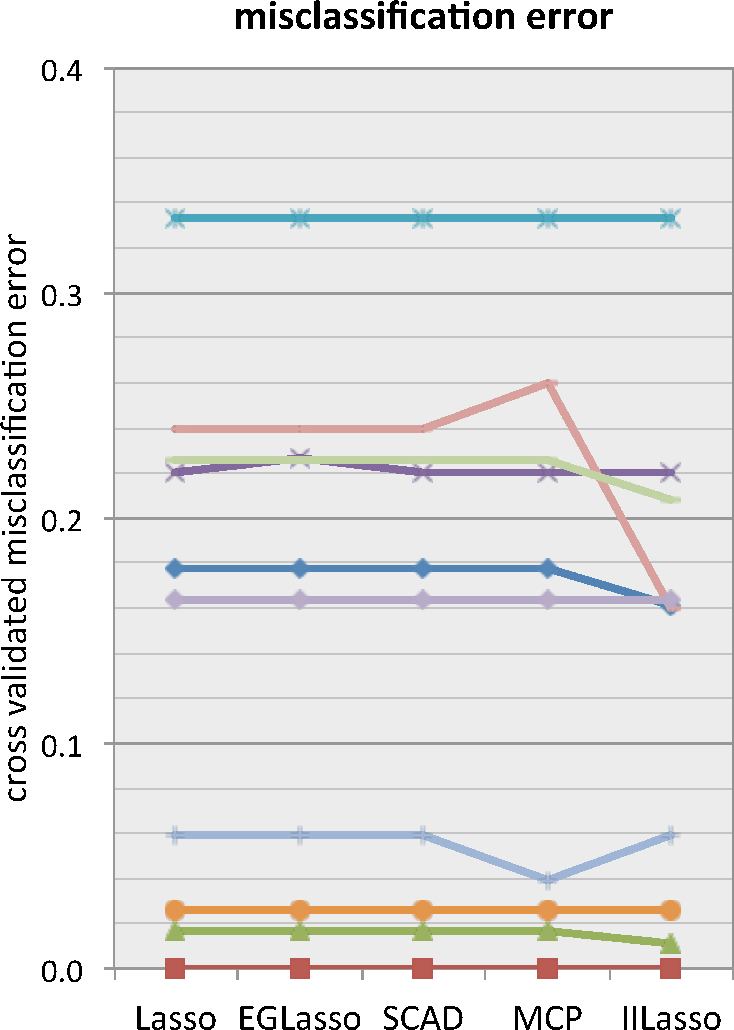}
		\end{minipage}
		\begin{minipage}[t]{0.3\linewidth}
			\includegraphics[keepaspectratio,scale=0.35]{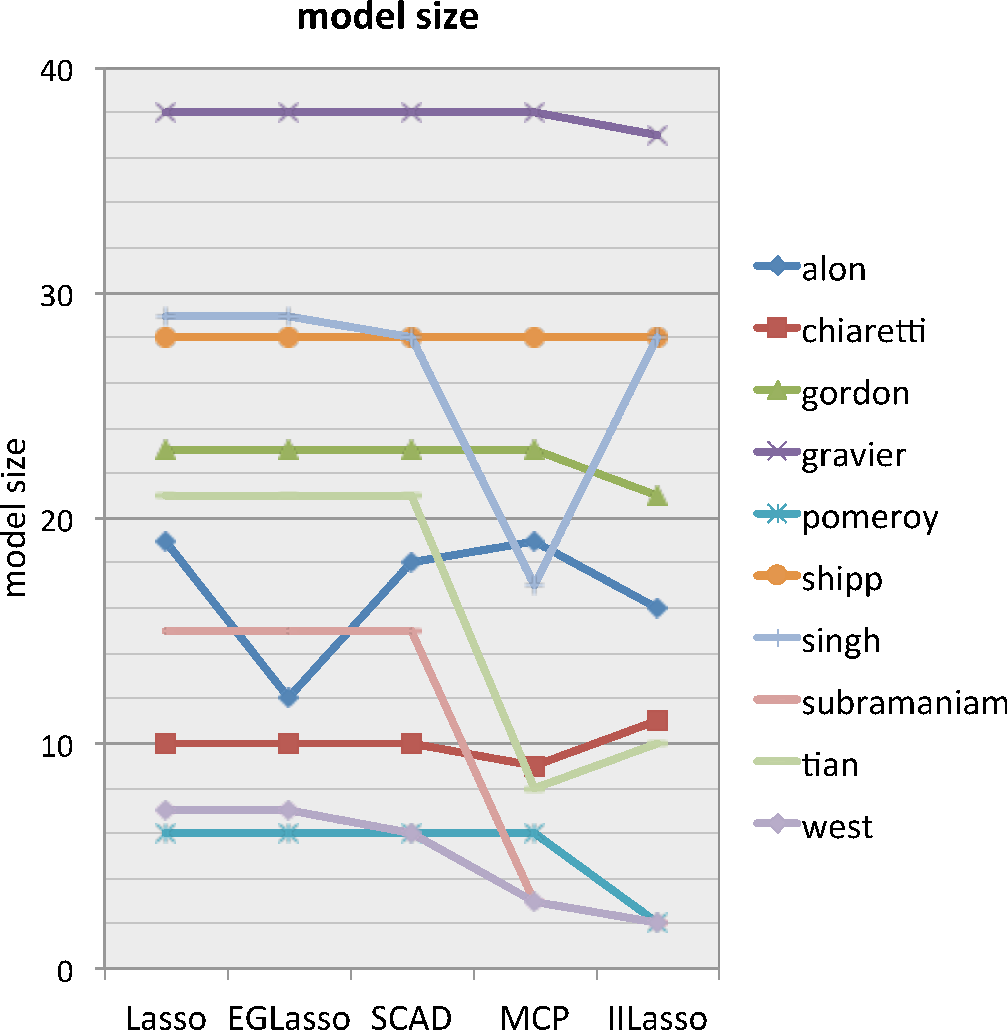}
		\end{minipage}
	\end{tabular}
	\caption{Results of 10 microarray datasets}
	\label{fig:realdata}
\end{figure*}

\section{Conclusion}
In this paper, we proposed a new regularization method ``IILasso''.
IILasso reduces correlations among the active variables, hence it is easy to decompose and interpret the model.
We showed that the sign recovery condition of IILasso is milder than that of Lasso when the true active variables are uncorrelated with each other.
The convergence rate of IILasso also has a better performance compared to that of Lasso.
In addition, we showed that every local optimum by coordinate descent algorithm has the same order convergence rate as the global optimum.
Finally, we verified the effectiveness of IILasso by synthetic and real data analyses using 10 gene expression data, and we saw that IILasso was superior 
in many cases on 
high-dimensional data.

\section*{Acknowledgement}
TS was partially supported by MEXT kakenhi (25730013, 25120012, 26280009, 15H01678 and 15H05707), JST-PRESTO and JST-CREST. HF was partially supported by MEXT kakenhi 17K00065.

\bibliographystyle{plain}
\bibliography{nips_2017_takada,nips_2017_suzuki}

\renewcommand{\thesection}{\Alph{section}}
\numberwithin{equation}{section}
\numberwithin{Theorem}{section}
\renewcommand{\thealgorithm}{\thesection.\arabic{algorithm}}
\setcounter{section}{0}
\setcounter{algorithm}{0}


\onecolumn

\section{Proof of Theorem \ref{theorem:sign}}
\begin{proof}
	By standard conditions for optimality, $\hat{\beta}$ is a critical point if and only if there exits a subgradient
	$\hat{z} \in \partial \|\hat{\beta}\|_1 := \{ \hat{z} \in \mathbb{R}^p | \hat{z_j} = \sgn(\hat{\beta}_j) \ {\rm for} \ \hat{\beta}_j \neq 0, \ |\hat{z_j}| \leq 1 \ {\rm otherwise} \}$ such that $\partial_{\hat{\beta}} L(\beta) = 0$.
	Because $\partial_{\beta} \frac{1}{2}|\beta|^\top R |\beta| = {\rm Diag}(R|\beta|) z$, the condition $\partial_{\hat{\beta}} L(\beta) = 0$ yields
	\begin{align}
	-\frac{1}{n}X^\top(y-X\hat{\beta}) + \lambda \hat{z} + \lambda \alpha {\rm Diag} \left( R |\hat{\beta}| \right) \hat{z} = 0. \label{eq:kkt}
	\end{align}
	Substituting $y = X\beta^* + \epsilon$ in (\ref{eq:kkt}), we have
	\begin{align}
	-\frac{1}{n}X^\top(X (\beta^* - \hat{\beta}) + \epsilon) +
	\lambda \hat{z} + \lambda \alpha{\rm Diag} \left(R |\hat{\beta}|\right) \hat{z} = 0.\label{eq:kkt2}
	\end{align}
	Let the true active set $S=\{1, \cdots, s\}$ and inactive set $S^c=\{s+1, \cdots, p\}$ without loss of generality, then (\ref{eq:kkt2}) is turned into
	\begin{align}
	\frac{1}{n} X_S^\top X_S \left( \hat{\beta}_S - \beta_S^* \right)
	+ \frac{1}{n} X_S^\top X_{S^c} \hat{\beta}_{S^c}
	- \frac{1}{n} X_S^\top \epsilon
	+ \lambda \hat{z}_S + \lambda \alpha{\rm Diag}\left(R_{SS}|\hat{\beta}_S|\right) \hat{z}_S = 0,\label{eq:kkts2} \\
	\frac{1}{n} X_{S^c}^\top X_S \left( \hat{\beta}_S - \beta_S^* \right)
	+ \frac{1}{n} X_{S^c}^\top X_{S^c} \hat{\beta}_{S^c}
	- \frac{1}{n} X_{S^c}^\top \epsilon
	+ \lambda \hat{z}_{S^c} + \lambda \alpha{\rm Diag}\left(R_{S^cS}|\hat{\beta}_S|\right) \hat{z}_{S^c}
	= 0.\label{eq:kktsc2}
	\end{align}
	Hence, there exists a critical point with correct sign recovery if and only if there exists $\hat{\beta}$ and $\hat{z}$ such that (\ref{eq:kkts2}), (\ref{eq:kktsc2}), $\hat{z} \in \partial \|\hat{\beta}\|_1$ and $\sgn(\hat{\beta}) = \sgn(\beta^*)$.
	The latter two conditions can be written as
	\begin{align}
	\hat{z}_S &= \sgn(\beta_S^*), \label{eq:z_s}\\
	|\hat{z}_{S^c}| &\leq 1, \label{eq:z_sc}\\
	\sgn(\hat{\beta}_S) &= \sgn(\beta_S^*), \label{eq:beta_s}\\
	\hat{\beta}_{S^c} &= 0. \label{eq:beta_sc}
	\end{align}
	The condition (\ref{eq:z_s}) and (\ref{eq:beta_sc}) yield
	\begin{align}
	\frac{1}{n} X_S^\top X_S \left( \hat{\beta}_S - \beta_S^* \right)
	- \frac{1}{n} X_S^\top \epsilon
	+ \lambda \sgn(\beta_S^*) + \lambda \alpha{\rm Diag}\left(R_{SS}|\hat{\beta}_S|\right) \sgn(\beta_S^*) = 0, \label{eq:kkts3} \\
	\frac{1}{n} X_{S^c}^\top X_S \left( \hat{\beta}_S - \beta_S^* \right)
	- \frac{1}{n} X_{S^c}^\top \epsilon
	+ \lambda \hat{z}_{S^c} + \lambda \alpha{\rm Diag}\left(R_{S^cS}|\hat{\beta}_S|\right) \hat{z}_{S^c}
	= 0.\label{eq:kktsc3}
	\end{align}
	Since
	\begin{align*}
	{\rm Diag}(R_{SS}|\hat{\beta}_S|)\sgn(\beta_S^*)
	=& {\rm Diag}(\sgn(\beta_S^*))R_{SS}|\hat{\beta}_S| \nonumber \\
	=& {\rm Diag}(\sgn(\beta_S^*))R_{SS}{\rm Diag}(\sgn(\beta_S^*))\hat{\beta}_S,
	\end{align*}
	(\ref{eq:kkts3}) can be rewritten as
	\begin{align*}
	U(\hat{\beta}_S - \beta_S^*) + V = 0,
	\end{align*}
	where
	\begin{align*}
	U &:= \frac{1}{n} X_S^\top X_S + \lambda\alpha {\rm Diag}(\sgn(\beta_S^*)) R_{SS} {\rm Diag}(\sgn(\beta_S^*)), \\
	V &:= \lambda \sgn(\beta_S^*) + \lambda \alpha{\rm Diag}(\sgn(\beta_S^*)) R_{SS} {\rm Diag}(\sgn(\beta_S^*)) \beta_S^* - \frac{1}{n} X_S^\top \epsilon.
	\end{align*}
	If we assume $U$ is invertible, we obtain
	\begin{align}
	\hat{\beta}_S = \beta_S^* - U^{-1} V. \label{eq:kkts4}
	\end{align}
	Substituting this in (\ref{eq:kktsc3}), we have
	\begin{align*}
	\frac{1}{n} X_{S^c}^\top X_S \left( -U^{-1} V \right)
	- \frac{1}{n} X_{S^c}^\top \epsilon
	+ \lambda \hat{z}_{S^c}
	+ \lambda \alpha{\rm Diag}\left( R_{S^cS}|\beta_S^* - U^{-1} V| \right) \hat{z}_{S^c}
	= 0,
	\end{align*}
	that is,
	\begin{align}
	\left( 1 + \alpha{\rm Diag}\left( R_{S^cS}|\beta_S^* - U^{-1} V| \right) \right) \lambda \hat{z}_{S^c}
	= \frac{1}{n} X_{S^c}^\top X_S U^{-1} V + \frac{1}{n} X_{S^c}^\top \epsilon. \label{eq:kktsc4}
	\end{align}
	Combining (\ref{eq:z_sc}), (\ref{eq:beta_s}), (\ref{eq:kkts4}) and (\ref{eq:kktsc4}), we have the following conditions:
	\begin{align*}
	\sgn(\beta_S^* - U^{-1} V) &= \sgn(\beta_S^*), \\
	\left| \frac{1}{n} X_{S^c}^\top X_S U^{-1} V + \frac{1}{n} X_{S^c}^\top \right|
	&\leq
	\lambda \left(1 + \alpha R_{S^cS}|\beta_S^* - U^{-1} V| \right).
	\end{align*}


\end{proof}


\section{Proof of Theorem \ref{th:OracleIneaulityBetahat}}
First, we prepare the following lemma.
\setcounter{Theorem}{0}
\begin{Lemma}
	\label{lemm:gamman}
	\rm
	Suppose that Assumption \ref{ass:Noise} and
	\begin{equation*}
	\frac{1}{n} \sum_{i=1}^n X_{ij}^2 \leq1~~~(\forall j=1,\dots,p),
	\end{equation*}
	are satisfied.
	For $\forall \delta > 0$, let $\gamma_n :=\gamma_n(\delta)$ be
	$$
	\gamma_n := \sigma  \sqrt{\frac{2\log(2p/\delta)}{n}}.
	$$
	Then, we have that
	$$
	P\left(\left\|\frac{1}{n} X^\top \epsilon \right\|_\infty \geq \gamma_n \right) \leq \delta.
	$$
\end{Lemma}
\begin{proof}
	The assertion can be shown in the standard way.
	First notice that
	\begin{align*}
	P\left( \left\|\frac{1}{n} X^\top \epsilon \right\|_\infty \geq \gamma \right)
	&
	= P\left( \max_{1\leq j \leq p} \left|\frac{1}{n} \sum_{i=1}^n \epsilon_i X_{ij}  \right| \geq \gamma \right) \\
	&
	= P\left( \bigcup_{1\leq j \leq p} \left\{ \left|\frac{1}{n} \sum_{i=1}^n \epsilon_i X_{ij}  \right| \geq \gamma \right\} \right) \\
	&
	\leq  \sum_{j=1}^p P\left(  \left|\frac{1}{n} \sum_{i=1}^n \epsilon_i X_{ij}  \right| \geq \gamma  \right)
	\leq
	p \max_{1\leq j \leq p} P\left(  \left|\frac{1}{n} \sum_{i=1}^n \epsilon_i X_{ij}  \right| \geq \gamma  \right).
	\end{align*}
	Since $\frac{1}{n}\sum_{i=1}^n X_{ij}^2 \leq 1$, $\xi_i = X_{ij} \epsilon_i$ satisfies $\EE[e^{t\xi_i}] \leq e^{\sigma^2 t^2/2} \ \forall t \in \Real$.
	Hence, applying Hoeffding's inequality, we obtain the assertion.
\end{proof}

Then, we derive Theorem \ref{th:OracleIneaulityBetahat}.
\begin{proof}
	By $\Llambda(\betahat) \leq \Llambda(\betastar)$ 
	and $y=X\betastar + \epsilon$, it holds that
	\begin{align}
	&
	\frac{1}{2n}\|X (\betahat - \betastar) - \epsilon \|_2^2
	+ \lambda_n \psi(\betahat) \leq
	\frac{1}{2n}\| \epsilon \|_2^2
	+ \lambda_n \psi(\betastar) \notag \\
	\Rightarrow~~
	&
	\frac{1}{2n}\|X (\betahat - \betastar)\|_2^2 + \lambda_n \psi(\betahat) \leq
	\frac{1}{n} \epsilon^\top X (\betahat - \betastar) + \lambda_n   \psi(\betastar),
	\label{eq:LassoBasicIneqBeforConcentrate}
	\end{align}
	where $\psi(\beta) = \lambda_n \left( \|\beta\|_1 + \frac{\alpha}{2} |\beta|^\top R |\beta| \right)$.
	By Lemma \ref{lemm:gamman}, it holds that
	\begin{align*}
	P\left(\left\|\frac{1}{n}X^\top \epsilon \right\|_\infty > \gamma_n \right) \leq \delta.
	\end{align*}
	Hereafter, we assume that the event $\{\left\|\frac{1}{n}X^\top \epsilon \right\|_\infty \leq \gamma_n \}$ is happening.

	Then, if $\gamma_n \leq \lambda_n / 3$, by (\ref{eq:LassoBasicIneqBeforConcentrate}),
	\begin{align}
	\frac{1}{2n}\|X (\betahat - \betastar)\|_2^2 + \lambda_n \psi(\betahat)
	&
	\leq
	\frac{1}{n} \|\epsilon^\top X\|_\infty \|\betastar - \betahat\|_1 + \lambda_n \psi(\betastar) \notag \\
	&
	\leq
	\gamma_n \|\betastar - \betahat\|_1 + \lambda_n \psi(\betastar)
	\leq \frac{1}{3} \lambda_n
	\|\betastar - \betahat\|_1 + \lambda_n \psi(\betastar).
	\label{eq:LassoConvRoot}
	\end{align}
	Since $$
	\|\betahat - \betastar\|_1 = \|\betahat_{S} - \betastar_{S}\|_1 + \|\betahat_{S^c} - \betastar_{S^c}\|_1 =
	\|\betahat_{S} - \betastar_{S}\|_1 + \|\betahat_{S^c}\|_1,
	$$
	and
	\begin{align*}
	|\betastar_S|^\top R_{SS} |\betastar_S| -
	|\betahat_S|^\top R_{SS} |\betahat_S|
	& \leq
	\sum_{(j,k) \in S \times S } R_{jk} |\betastar_j \betastar_k- \betahat_j \betahat_k| \\
	&
	\leq
	2 \sum_{(j,k) \in S \times S } R_{jk} |\betastar_j (\betastar_k - \betahat_k) |
	+
	\sum_{(j,k) \in S \times S } R_{jk} |(\betastar_j - \betahat_j)(\betastar_k - \betahat_k)| \\
	&
	=
	2 |\betastar_S|^\top R_{SS}  |\betastar_S - \betahat_S |
	+
	|\betastar_S - \betahat_S |  ^\top  R_{SS}  |\betastar_S - \betahat_S | \\
	&
	\leq
	2 \|R_{SS}  |\betastar_S|\|_{\infty} \| \betastar_S - \betahat_S \|_1
	+
	D  \|\betastar_S - \betahat_S \|_1^2,
	\end{align*}
	we obtain that
	\begin{align}
	& \frac{1}{2 n}\|X (\betahat - \betastar)\|_2^2 + \lambda_n
	\left( \|\betahat_S\|_1 +   \|\betahat_{S^c}\|_1 + \frac{\alpha}{2}|\betahat_S|^\top R_{SS}|\betahat_S|
	+ \frac{\alpha}{2}\sum_{(j,k) \notin S\times S} R_{jk} |\betahat_j \betahat_k| \right) \notag \\
	&
	\leq
	\frac{1}{3} \lambda_n (\|\betahat_{S} - \betastar_{S}\|_1 + \|\betahat_{S^c}\|_1)
	+ \lambda_n \left(\|\betastar_S\|_1 + \frac{\alpha}{2}|\betastar_S|^\top R_{SS} |\betastar_S|\right)  \notag
	\\
	\Rightarrow &
	\frac{1}{2n}\|X (\betahat - \betastar)\|_2^2 + \lambda_n
	\left( \frac{2}{3} \|\betahat_{S^c}\|_1
	+ \frac{\alpha}{2}\sum_{(j,k) \notin S\times S} R_{jk}|\betahat_j \betahat_k| \right) \notag \\
	&
	\leq
	\frac{1}{3} \lambda_n \|\betahat_{S} - \betastar_{S}\|_1
	+ \lambda_n \left(\|\betastar_S\|_1 - \|\betahat_S\|_1
	+ \alpha \|R_{SS}  |\betastar_S|\|_{\infty} \| \betastar_S - \betahat_S \|_1
	+ \frac{\alpha D}{2}  \|\betastar_S - \betahat_S \|_1^2\right) \notag \\
	\Rightarrow &
	\frac{1}{2n}\|X (\betahat - \betastar)\|_2^2 + \lambda_n
	\left( \frac{2}{3} \|\betahat_{S^c}\|_1
	+ \frac{\alpha}{2} \sum_{(j,k) \notin S\times S} R_{jk} |\betahat_j \betahat_k| \right) \notag \\
	&
	\leq
	\lambda_n \left(\frac{4}{3} \|\betahat_{S} - \betastar_{S}\|_1
	+ \alpha \|R_{SS}  |\betastar_S|\|_{\infty} \| \betastar_S - \betahat_S \|_1
	+ \frac{\alpha D}{2} \|\betastar_S - \betahat_S \|_1^2 \right).
	\label{eq:FirstBoundRight}
	\end{align}

	On the other hand, (\ref{eq:LassoConvRoot}) also gives
	\begin{align*}
	&  \|\betahat_S \|_1 + \|\betahat_{S^c} \|_1
	\leq
	\frac{1}{3} (\|\betahat_{S} - \betastar_{S}\|_1 + \|\betahat_{S^c}\|_1)
	+ \|\betastar_S\|_1 + \frac{\alpha}{2} |\betastar_S|^\top R_{SS} |\betastar_S| \\
	\Rightarrow ~~~
	& \frac{2}{3} \| \betahat_S - \betastar_S \|_1 + \frac{2}{3}\| \betahat_{S^c} \|_1 \leq 2 \| \betastar_S\|_1 + \frac{\alpha}{2} |\betastar_S|^\top R_{SS} |\betastar_S| \\
	\Rightarrow ~~~
	& \| \betahat_S - \betastar_S \|_1 \leq 3 \| \betastar_S\|_1 + \frac{3}{4} \alpha |\betastar_S|^\top R_{SS} |\betastar_S|\\
	\Rightarrow ~~~
	& \| \betahat_S - \betastar_S \|_1 \leq
	\left( 3 + \frac{3}{4} \alpha \| R_{SS} |\betastar_S| \|_{\infty} \right) \|\betastar_{S}\|_1 .
	\end{align*}

	Therefore, (\ref{eq:FirstBoundRight}) gives
	\begin{align}
	& \frac{2}{3} \|\betahat_{S^c}\|_1 + \frac{\alpha}{2} \sum_{(j,k) \notin S\times S} R_{jk} |\betahat_j \betahat_k| \nonumber \\
	& \leq
	\left(\frac{4}{3} + \alpha \|R_{SS}  |\betastar_S|\|_{\infty}
	+ \frac{3}{2}\alpha D \|\betastar_{S}\|_1
	\left(1 + \frac{\alpha}{4} \| R_{SS} |\betastar_S| \|_{\infty}\right) \right) \|\betahat_{S} - \betastar_{S}\|_1.
	\label{eq:BCconstraintcheck1}
	\end{align}

	The second term of the left side is evaluated as 
	\begin{align*}
	\sum_{(j,k) \notin S\times S} R_{jk} |\betahat_j \betahat_k| =
	& \sum_{j \in S^c, k \in S^c} R_{jk} |\betahat_j \betahat_k|
	 + 2 \sum_{j \in S, k \in S^c} R_{jk} |(\betahat_j-\betastar_S+\betastar_S) \betahat_k| \\
	=
	& |\betahat_{S^c}|^\top R_{S^c S^c} |\betahat_{S^c}|
	+ 2 |\betahat_{S^c}|^\top R_{S^c S} |\betahat_S - \betastar_S + \betastar_S|.
	\end{align*}

	Hence, (\ref{eq:BCconstraintcheck1}) gives
	\begin{align}
	&\frac{2}{3} \| \betahat_{S^c} \|_1 + \frac{\alpha}{2} |\betahat_{S^c}|^\top R_{S^cS^c}|\betahat_{S^c}| + \alpha |\betahat_{S^c}|^\top R_{S^cS} |\betahat_S - \betastar_S + \betastar_S| \nonumber \\
	\leq & \left( \frac{4}{3} + \alpha \| R_{SS}|\betastar_S| \|_\infty + \frac{3}{2} \alpha D \| \betastar_S\|_1 \left( 1 + \frac{\alpha}{4} \| R_{SS}|\betastar_S| \|_\infty \right) \right) \| \betahat_S - \betastar_S \|_1 \nonumber \\
	\Rightarrow ~~~ &
	\| \betahat_{S^c} \|_1 + \frac{3}{4} \alpha |\betahat_{S^c}|^\top R_{S^cS^c}|\betahat_{S^c}| + \frac{3}{2} \alpha |\betahat_{S^c}|^\top R_{S^cS} |\betahat_S - \betastar_S + \betastar_S| \nonumber \\
	\leq & \left( 2 + \frac{15}{4} \alpha D \| \betastar_S \|_1 + \frac{9}{16} (\alpha D \| \betastar_S\|_1)^2 \right) \| \betahat_S - \betastar_S \|_1 . \label{eq:BetaStrictConstraint}
	\end{align}

	If $\alpha \leq \frac{1}{4 D \| \betastar_S \|_1}$, we have
	\begin{align*}
	\| \betahat_{S^c} \|_1 + \frac{3}{4} \alpha |\betahat_{S^c}|^\top R_{S^cS^c}|\betahat_{S^c}| + \frac{3}{2} \alpha |\betahat_{S^c}|^\top R_{S^cS} |\betahat_S - \betastar_S + \betastar_S|
	\leq  3 \| \betahat_S - \betastar_S \|_1 .
	\end{align*}

	Therefore, we can see that
	$$
	\Delta \beta \in \calB(S, C, C'),
	$$
	where $\Delta \beta = \betahat - \betastar$, $C=3$ and $C'=\frac{3}{2}$.
	By applying the definition of $\phi_{\GRE}$ to (\ref{eq:FirstBoundRight}), it holds that
	\begin{align*}
	& \frac{\phi_{\GRE}}{2} \|\betahat - \betastar\|_2^2
	\leq \lambda_n \left( \frac{4}{3} + \frac{5}{2} \alpha D \| \betastar_S \|_1 + \frac{3}{8} (\alpha D \| \betastar_S\|_1)^2 \right) \|\betahat_S - \betastar_S\|_1
	\end{align*}
	Because $\|\betahat_S - \betastar_S\|_1^2 \leq s \|\betahat_S - \betastar_S\|_2^2$, we have
	\begin{align}
	&\|\betahat - \betastar\|_2
	\leq  \frac{\left( \frac{8}{3} + 5 \alpha D \| \betastar_S \|_1 + \frac{3}{4} (\alpha D \| \betastar_S\|_1)^2 \right) \sqrt{s} \lambda_n}{\phi_{\GRE}} \notag \\
	\Rightarrow ~~~ &
	\|\betahat - \betastar\|_2^2
	\leq  \frac{\left( \frac{8}{3} + 5 \alpha D \| \betastar_S \|_1 + \frac{3}{4} (\alpha D \| \betastar_S\|_1)^2 \right)^2 s \lambda_n^2}{\phi_{\GRE}^2}
	\leq  \frac{16 s \lambda_n^2}{\phi_{\GRE}^2} \label{eq:BetaBound}
	\end{align}
	This concludes the assertion.
\end{proof}

\section{Corollary of Theorem \ref{th:OracleIneaulityBetahat}}

For comparison with IILasso and Lasso, we use the following a little bit stricter bound.

\begin{Corollary}
	Suppose the same assumption of Theorem \ref{th:OracleIneaulityBetahat}
	except for $\alpha \leq \frac{1}{4 D \| \betastar_S \|_1}$ and Assumption GRE$(S,3,\frac{3}{2})$.
	Instead, suppose that Assumption GRE$(S, C, \frac{3}{2})$ (Definition \ref{eq:AssGRE}) where $C = 2+\frac{15}{4}\alpha D\|\betastar_S\|_1+\frac{9}{16}(\alpha D\|\betastar_S\|_1)^2$ is satisfied.
	Then, it holds that
	$$
	\|\betahat - \betastar\|_2^2
	\leq  \frac{\left( \frac{8}{3} + 5 \alpha D \| \betastar_S \|_1 + \frac{3}{4} (\alpha D \| \betastar_S\|_1)^2 \right)^2 s \lambda_n^2}{\phi_{\GRE}^2},
	$$
	with probability $1-\delta$.
\end{Corollary}

\begin{proof}
	This is derived basically in the same way as Theorem \ref{th:OracleIneaulityBetahat}.
	From \eqref{eq:BetaStrictConstraint}, we can see directly that
	$$
	\Delta \beta \in \calB(S, C, C'),
	$$
	where $\Delta \beta = \betahat - \betastar$, $C = 2+\frac{15}{4}\alpha D\|\betastar_S\|_1+\frac{9}{16}(\alpha D\|\betastar_S\|_1)^2$ and $C'=\frac{3}{2}$.
	This and \eqref{eq:BetaBound} concludes the assertion.
\end{proof}

From this corollary, we can compare Lasso and IILasso with $R_{SS}=O$.
\begin{itemize}
	\item If $\alpha=0$, we have
	$$
	\|\betahat - \betastar\|_2^2
	\leq  \frac{64 s \lambda_n^2}{9 \phi_{\GRE}^2},
	$$
	with $\calB(S, C, C')$ where $C=2$ and $C'=0$.
	This is a standard Lasso result.
	\item If $D=0$, we have
	$$
	\|\betahat - \betastar\|_2^2
	\leq  \frac{64 s \lambda_n^2}{9 \phi_{\GRE}^2},
	$$
	with $\calB(S, C, C')$ where $C=2$ and $C'=\frac{3}{2}$.
	Since $\phi_{\rm GRE}$ is the minimum eigenvalue restricted by $\calB(S, C, C')$, $\phi_{\rm GRE}$ of IILasso is larger than that of Lasso.
\end{itemize}

%

\section{Proof of Theorem \ref{th:LocalOptimality}}

\begin{proof}
	Let
	$$
	\betacheck := \mathop{\arg \min}_{\beta \in \Real^p : \beta_{S^c}=0}
	\|y - X\beta\|_2^2.
	$$
	That is, $\betacheck$ is the least squares estimator with the true non-zero coefficients.
	Let $\betatilde$ be a local optimal solution.
	For $0 < h < 1$, letting $\beta(h) := \betatilde + h(\betacheck - \betatilde)$, then it holds that
	\begin{align}
	\Llambda(\beta(h)) - \Llambda(\betatilde)
	= & \frac{h^2-2 h}{2n} \|X(\betatilde - \betacheck)\|_2^2  - \frac{h}{n} (X\betacheck - y)^\top X(\betatilde-\betacheck) \notag \\
	& + \lambda_n (\|\beta(h)\|_1 - \|\betatilde\|_1) +
	\frac{\lambda_n \alpha}{2} ( |\beta(h)|^\top R |\beta(h)| -  |\betatilde|^\top R |\betatilde| ).
	\label{eq:LhatbetahmLhatbeta}
	\end{align}

	First we evaluate the term $ \frac{1}{n} (X\betacheck - y)^\top X(\betatilde-\betacheck)=
	\frac{1}{n} (X\betacheck - y)^\top X_S(\betatilde_S-\betacheck_S)
	+
	\frac{1}{n} (X\betacheck - y)^\top X_{S^c}(\betatilde_{S^c}-\betacheck_{S^c})$ as follows: \\
	\noindent (1)
	Since $\betacheck$ is the least squares estimator and
	$\frac{1}{n} X_S^\top X_S$ is invertible by the assumption,
	we have
	$$
	\betacheck_S = (X_S^\top X_S)^{-1} X_S^\top y,~~~\betacheck_{S^c} = 0.
	$$
	Therefore,
	\begin{align*}
	\frac{1}{n}X_S^\top (X \betacheck - y)
	& = \frac{1}{n} X_S^\top (X_S (X_S^\top X_S)^{-1} X_S^\top - I) y.
	\end{align*}
	Here, $I - X_S (X_S^\top X_S)^\top X_S^\top$ is the projection matrix to the
	orthogonal complement of the image of $(X_S^\top X_S)^\top$.
	Hence, $\frac{1}{n} (X\betacheck - y)^\top X_S(\betatilde_S-\betacheck_S)=0$.
	\\
	\noindent (2)
	Noticing that
	\begin{align*}
	\frac{1}{n}X_{S^c}^\top (X \betacheck - y)  &
	=
	- \frac{1}{n}X_{S^c}^\top
	(I - X_S (X_S^\top X_S)^{-1} X_S^\top)  y  \\
	&
	=
	- \frac{1}{n}X_{S^c}^\top
	(I - X_S (X_S^\top X_S)^{-1} X_S^\top)  (X_S \betastar_S + \epsilon) \\
	& =
	- \frac{1}{n}X_{S^c}^\top
	(I - X_S (X_S^\top X_S)^{-1} X_S^\top) \epsilon ,
	\end{align*}
	where we used $ (I - X_S (X_S^\top X_S)^{-1} X_S^\top)X_{S^c}=0$ in the last line.
	Because $(I - X_S (X_S^\top X_S)^\top X_S^\top)$ is a projection matrix, we have
	$\|(I - X_S (X_S^\top X_S)^{-1} X_S^\top) X_j\|_2^2 \leq \|X_j\|_2^2$.
	This and Lemma \ref{lemm:gamman} gives
	$$
	\left\| \frac{1}{n}X_{S^c}^\top (X \betacheck - y)\right\|_\infty \leq \gamma_n,
	$$
	with probability $1- \delta$. Hence, let $V:=\supp(\betatilde) \backslash S$, then we have
	$$
	\left| \frac{1}{n}(\betatilde_{S^c} - \betacheck_{S^c})^\top
	X_{S^c}^\top (X \betacheck - y) \right|
	\leq \gamma_n \|\betatilde_{S^c} - \betacheck_{S^c}\|_1 =
	\gamma_n \|\betatilde_{\Sc}\|_1.
	$$
	where we used the assumption 
	$\Sc \subseteq S^c$
	and $\betacheck_{\Sc}=0$.

	Combining these inequalities and the assumption $\lambda_n \geq \gamma_n$, we have that
	\begin{align}
	\left| \frac{1}{n} (X\betacheck - y)^\top X(\betatilde-\betacheck) \right|
	\leq
	& \lambda_n \| \betatilde_{\Sc}  \|_1.
	\label{eq:XbetaresBound}
	\end{align}
	As for the regularization term, we evaluate each term of
$\lambda_n (\|\beta(h)\|_1 - \|\betatilde\|_1) +
	\frac{\lambda_n}{2} ( |\beta(h)|^\top R |\beta(h)| -  |\betatilde|^\top R |\betatilde| )$ in the following.

(i) Evaluation of $\|\beta(h)\|_1 - \|\betatilde\|_1$.
Because of the definition of $\beta(h)$, it holds that
\begin{align}
\|\beta(h)\|_1 - \|\betatilde\|_1
& = \|\betatilde + h(\betacheck - \betatilde)\|_1 - \|\betatilde\|_1 \notag \\
& = \|\betatilde_S + h(\betacheck_S - \betatilde_S)\|_1 - \|\betatilde_S\|_1
+\|\betatilde_V + h(\betacheck_V - \betatilde_V)\|_1 - \|\betatilde_V\|_1 \notag \\
& = \|\betatilde_S + h(\betacheck_S - \betatilde_S)\|_1 - \|\betatilde_S\|_1
+(1-h) \| \betatilde_V \|_1 - \|\betatilde_V\|_1 \notag \\
&\leq
h \|\betacheck_S - \betatilde_S\|_1 -h \| \betatilde_V \|_1.
\label{eq:betahLonediff}
\end{align}

(ii) Evaluation of $ |\beta(h)|^\top R |\beta(h)| -  |\betatilde|^\top R |\betatilde|$.
Note that
\begin{align}
&|\beta(h)_j| R_{jk}|\beta(h)_k| - |\betatilde_j| R_{jk} |\betatilde_k| \notag \\
& =
|(1-h)\betatilde_j + h \betacheck_j| R_{jk}|(1-h)\betatilde_k + h \betacheck_k| - |\betatilde_j| R_{jk} |\betatilde_k|  \notag\\
& \leq
(1-h)^2|\betatilde_j |R_{jk}|\betatilde_k|+ h (1-h)(|\betacheck_j| R_{jk}|\betatilde_k| + |\betatilde_j| R_{jk}|\betacheck_k|)  \notag\\
& ~~~~+ h^2 |\betacheck_j|R_{jk}|\betacheck_k| - |\betatilde_j| R_{jk} |\betatilde_k|  \notag\\
& =
-2h|\betatilde_j |R_{jk}|\betatilde_k|  + h (|\betacheck_j| R_{jk}|\betatilde_k| + |\betatilde_j| R_{jk}|\betacheck_k|)
+O(h^2) \notag \\
& =
 h [(|\betacheck_j| -|\betatilde_j |) R_{jk}|\betatilde_k| + |\betatilde_j| R_{jk} (|\betacheck_k|  -|\betatilde_k|)]
+O(h^2).
\label{eq:RbetahRbetaRoughEval}
\end{align}
If $j,k \in S$, then the right hand side of Eq. \eqref{eq:RbetahRbetaRoughEval} is bounded by
\begin{align*}
& 
 h (|\betacheck_j  - \betatilde_j | R_{jk}|\betacheck_k  - \betatilde_k| + |\betacheck_j  - \betatilde_j| R_{jk} |\betacheck_k  -\betatilde_k| )  \\
&~~~+
 h (|\betacheck_j  - \betatilde_j | R_{jk}|\betacheck_k | + |\betacheck_j  | R_{jk} |\betacheck_k  -\betatilde_k| )
+O(h^2).
\end{align*}
If $j \in V$ and $k \in S$, then the right hand side of Eq. \eqref{eq:RbetahRbetaRoughEval} is bounded by
$$
 h  |\betatilde_j| R_{jk} (|\betacheck_k|  -|\betatilde_k|) +O(h^2) \leq
 h  |\betatilde_j| R_{jk} |\betacheck_k  -\betatilde_k| +O(h^2).
$$
If $j \in V$ and $k \in V$, then the right hand side of Eq. \eqref{eq:RbetahRbetaRoughEval} is bounded by
$$
0+O(h^2) = O(h^2).
$$
Based on these evaluations, we have
\begin{align*}
&  |\beta(h)|^\top R |\beta(h)| -  |\betatilde|^\top R |\betatilde| \\
\leq
&
2 h \left( |\betacheck_S - \betatilde_S|^\top R_{SS}|\betacheck_S - \betatilde_S|
+ |\betacheck_S - \betatilde_S|^\top R_{SS}|\betacheck_S|
+ |\betatilde_V|^\top R_{VS} |\betacheck_S - \betatilde_S|
\right)  + O(h^2) \\
\leq
&
2 h \left(  |\betacheck - \betatilde|^\top R|\betacheck - \betatilde|
+  |\betacheck_S - \betatilde_S|^\top R_{SS}|\betacheck_S|
\right)  + O(h^2) \\
\leq
&
2h \Ddash (  \|\betacheck - \betatilde\|_2^2 +  \|\betacheck\|_2 \|\betacheck_S - \betatilde_S\|_2) + O(h^2).
\end{align*}
Here, we will show later in Eq. \eqref{eq:betacheckbetastarerror} that
$\|\betacheck - \betastar\|_2 \leq \sqrt{s} \lambda_n/\phi$,
and thus it follows that
$$\|\betacheck\|_2  \leq \|\betastar\|_2 +  \sqrt{s} \lambda_n/\phi.$$
Therefore, we obtain that
\begin{align}
&  |\beta(h)|^\top R |\beta(h)| -  |\betatilde|^\top R |\betatilde| \notag\\
& \leq
2h \Ddash \left( \|\betacheck - \betatilde\|_2^2 +  (\|\betastar\|_2 +  \sqrt{s} \lambda_n/\phi) \|\betacheck_S - \betatilde_S\|_2 \right) + O(h^2).
\label{eq:Rbetahbetadiff}
\end{align}

%

	Applying the inequalities \eqref{eq:XbetaresBound}, \eqref{eq:betahLonediff} and \eqref{eq:Rbetahbetadiff} to \eqref{eq:LhatbetahmLhatbeta} yields that
	\begin{align*}
	& \Llambda(\beta(h)) - \Llambda(\betatilde)\\
	\leq
	& h \Big\{
	-\frac{1}{n} \|X(\betacheck - \betatilde)\|_2^2  + \lambda_n \|\betatilde_S - \betacheck_S\|_1
	- (\lambda_n - \gamma_n) \|\betatilde_{\Sc}\|_1
	\\
&
+\lambda_n \alpha \Ddash [ \|\betacheck - \betatilde\|_2^2 +  (\|\betastar\|_2 +  \sqrt{s} \lambda_n/\phi) \|\betacheck_S - \betatilde_S\|_2]
	\Big\}+
	O(h^2) \\
	\leq
	& h \Big\{
	-\phi \|\betacheck - \betatilde\|_2^2  + \lambda_n \|\betatilde_S - \betacheck_S\|_1
	\\
&
+\lambda_n \alpha \Ddash [ \|\betacheck - \betatilde\|_2^2 +  (\|\betastar\|_2 +  \sqrt{s} \lambda_n/\phi) \|\betacheck_S - \betatilde_S\|_2]
	\Big\}+
	O(h^2) \\
	\leq & h \Bigg\{ \left(-\phi + \lambda_n \alpha \Ddash \right) \|\betacheck - \betatilde \|_2^2   \\
&	+ \lambda_n \left( \|\betatilde_S - \betacheck_S\|_1
+ \alpha \Ddash (\|\betastar\|_2 +  \sqrt{s} \lambda_n / \phi) \|\betacheck_S - \betatilde_S\|_2 \right)
	\Bigg\}
	+ O(h^2),
	\end{align*}
	where we used the assumption $\lambda_n > \gamma_n$ in the second inequality.

	Since we have assumed $\alpha < \min \left\{ \frac{\sqrt{s}}{2\Ddash\| \betastar \|_2}, \frac{\phi}{2\Ddash\lambda_n}\right\}$,
	the right hand side is further bounded by
	\begin{align*}
	&
	h
	\left\{
	- \frac{\phi}{2} \|\betacheck - \betatilde\|_2^2
	+ 2 \lambda_n \sqrt{s} \|\betacheck_S - \betatilde_S\|_2
	\right\}
	+ O(h^2).
	\end{align*}
	Because of this,
	if $\|\betacheck - \betatilde\|_2 > \frac{4 \sqrt{s} \lambda_n}{\phi}$, then the first term
	becomes negative, and we conclude that, for sufficiently small $\eta > 0$, it holds that
	$$
	\Llambda(\beta(h)) < \Llambda(\betatilde),
	$$
	for all $0 < h < \eta $. In other word, $\betatilde$ is not a local optimal solution. Therefore,
	we must have
	$$
	\|\betacheck - \betatilde\|_2 \leq \frac{4 \sqrt{s} \lambda_n}{\phi}
	$$
	Finally, notice that $\|\betatilde - \betastar\|_2^2 \leq
	(\|\betatilde - \betacheck\|_2 + \|\betastar - \betacheck\|_2)^2$ and
	\begin{align}
	\|\betacheck - \betastar\|_2^2
	& = \|(X_S^\top X_S)^{-1} X_S^\top y - \betastar_S\|_2^2
	=\|(X_S^\top X_S)^{-1} X_S^\top (X_S \betastar_S + \epsilon) - \betastar_S\|_2^2 \notag \\
	&
	=\| (X_S^\top X_S)^{-1} X_S^\top  \epsilon\|_2^2
	\leq \phi^{-2} \|\frac{1}{n}X_S^\top  \epsilon\|_2^2
	\leq \phi^{-2} s \gamma_n^2
	\leq \phi^{-2} s \lambda_n^2,
	\label{eq:betacheckbetastarerror}
	\end{align}
	which concludes the assertion.
\end{proof}

\section{Optimization for Logistic Regression}
We derive coordinate descent algorithm of IILasso for the binary objective variable.
The objective function is
\begin{align*}
L(\beta) =
-\frac{1}{n} \sum_i \left(y_i X^i \beta - \log (1 + \exp(X^i \beta))\right) + \lambda \left( \|\beta\|_1 + \frac{\alpha}{2} |\beta|^\top R |\beta| \right),
\end{align*}
where $X^i$ is the i-th row of $X=[1, X_1, \cdots, X_p]$ and $\beta=[\beta_0, \beta_1, \cdots, \beta_p]$.
Forming a quadratic approximation with the current estimate $\betabar$, we have
\begin{align*}
\bar{L}(\beta) =
-\frac{1}{2n} \sum_{i=1}^n w_i(z_i - X^i\beta)^2 + C(\betabar) + \lambda \left( \|\beta\|_1 + \frac{\alpha}{2} |\beta|^\top R |\beta| \right), \label{eq:lr}
\end{align*}
where
\begin{align*}
z_i &= X^i \betabar + \frac{y_i - \bar{p}(X^i)}{\bar{p}(X^i)(1-\bar{p}(X^i))}, \\
w_i &= \bar{p}(X^i)(1-\bar{p}(X^i)), \\
\bar{p}(X^i) &= \frac{1}{1+\exp(-X^i\betabar)}.
\end{align*}
To derive the update equation, when $\beta_j \neq 0$, differentiating the quadratic objective function with respect to $\beta_j$ yields
\begin{align*}
\partial_{\beta_j} \bar{L}(\beta) =
& -\frac{1}{n} \sum_{i=1}^n w_i(z_i - X^i\beta)X_{ij} + \lambda \left( \sgn(\beta_j) + \alpha R_j^\top |\beta| \sgn(\beta_j) \right) \\
=& -\frac{1}{n} \sum_{i=1}^n w_i \left(z_i - X_{i, -j}\beta_{-j}\right) X_{ij}
+ \left( \frac{1}{n} \sum_{i=1}^n w_i X_{ij}^2
+ \lambda R_{jj} \right) \beta_j + \lambda \left( 1 + \alpha R_{j, -j} |\beta_{-j}| \right) \sgn(\beta_j).
\end{align*}
This yields
\begin{align*}
\beta_j \leftarrow \frac{1}{\frac{1}{n} \sum_{i=1}^n w_i X_{ij}^2
	+ \lambda \alpha R_{jj}}
S\left(\frac{1}{n} \sum_{i=1}^n w_i \left(z_i - X_{i, -j}\beta_{-j}\right) X_{ij}, \
\lambda \left(1 + \alpha R_{j, -j} |\beta_{-j}| \right) \right).
\end{align*}
These procedures amount to a sequence of nested loops.
The whole algorithm is described in Algorithm \ref{alg:logistic}.

\begin{algorithm}[t]
	\caption{CDA for Logistic IILasso}
	\label{alg:logistic}
	\begin{algorithmic}
		\FOR{$\lambda = \lambda_{\max}, \cdots, \lambda_{\min}$}
		\STATE initialize $\beta$
		\WHILE{until convergence}
		\STATE update the quadratic approximation using the current parameters $\betabar$
		\WHILE{until convergence}
		\FOR{$j = 1, \cdots, p$}
		\STATE $\beta_j \leftarrow \frac{1}{\frac{1}{n} \sum_{i=1}^n w_i X_{ij}^2
			+ \lambda \alpha R_{jj}}
		S\left(\frac{1}{n} \sum_{i=1}^n w_i \left(z_i - X_{i, -j}\beta_{-j}\right) X_{ij}, \
		\lambda \left(1 + \alpha R_{j, -j} |\beta_{-j}| \right) \right)$
		\ENDFOR
		\ENDWHILE
		\ENDWHILE
		\ENDFOR
	\end{algorithmic}
\end{algorithm}

\end{document}